\documentclass{article}
\pdfoutput=1

\usepackage{microtype}
\usepackage{graphicx}
\usepackage{subfigure}
\usepackage{booktabs} % for professional tables

\usepackage{hyperref}

\usepackage[accepted]{icml2019}

\usepackage[utf8]{inputenc} % allow utf-8 input
\usepackage[T1]{fontenc}    % use 8-bit T1 fonts
\usepackage{url}            % simple URL typesetting
\usepackage{amsfonts}       % blackboard math symbols
\usepackage{nicefrac}       % compact symbols for 1/2, etc.
\usepackage{epsf}
\usepackage{epsfig}
\usepackage{amsmath}
\usepackage{amssymb}
\usepackage{amsthm}
\usepackage{mathtools}
\usepackage{multirow}
\usepackage{multicol}
\usepackage{xspace}
\usepackage{listings}
\usepackage{xcolor}
\usepackage{enumitem}

\icmltitlerunning{Zeno: Distributed Stochastic Gradient Descent with Suspicion-based Fault-tolerance}

\begin{document}

\def\Blue{\color{blue}}
\def\Purple{\color{purple}}

\def\A{{\bf A}}
\def\a{{\bf a}}
\def\B{{\bf B}}
\def\C{{\bf C}}
\def\c{{\bf c}}
\def\D{{\bf D}}
\def\d{{\bf d}}
\def\F{{\bf F}}
\def\e{{\bf e}}
\def\f{{\bf f}}
\def\G{{\bf G}}
\def\H{{\bf H}}
\def\I{{\bf I}}
\def\K{{\bf K}}
\def\L{{\bf L}}
\def\M{{\bf M}}
\def\m{{\bf m}}
\def\N{{\bf N}}
\def\n{{\bf n}}
\def\Q{{\bf Q}}
\def\q{{\bf q}}
\def\S{{\bf S}}
\def\s{{\bf s}}
\def\T{{\bf T}}
\def\U{{\bf U}}
\def\u{{\bf u}}
\def\V{{\bf V}}
\def\v{{\bf v}}
\def\W{{\bf W}}
\def\w{{\bf w}}
\def\X{{\bf X}}
\def\x{{\bf x}}
\def\Y{{\bf Y}}
\def\y{{\bf y}}
\def\Z{{\bf Z}}
\def\z{{\bf z}}
\def\0{{\bf 0}}
\def\1{{\bf 1}}

\def\AM{{\mathcal A}}
\def\CM{{\mathcal C}}
\def\DM{{\mathcal D}}
\def\GM{{\mathcal G}}
\def\FM{{\mathcal F}}
\def\IM{{\mathcal I}}
\def\NM{{\mathcal N}}
\def\OM{{\mathcal O}}
\def\SM{{\mathcal S}}
\def\TM{{\mathcal T}}
\def\UM{{\mathcal U}}
\def\XM{{\mathcal X}}
\def\YM{{\mathcal Y}}
\def\RB{{\mathbb R}}

\def\TX{\tilde{\bf X}}
\def\tx{\tilde{\bf x}}
\def\ty{\tilde{\bf y}}
\def\TZ{\tilde{\bf Z}}
\def\tz{\tilde{\bf z}}
\def\hd{\hat{d}}
\def\HD{\hat{\bf D}}
\def\hx{\hat{\bf x}}
\def\TD{\tilde{\Delta}}

\def\alp{\mbox{\boldmath$\alpha$\unboldmath}}
\def\bet{\mbox{\boldmath$\beta$\unboldmath}}
\def\epsi{\mbox{\boldmath$\epsilon$\unboldmath}}
\def\etab{\mbox{\boldmath$\eta$\unboldmath}}
\def\ph{\mbox{\boldmath$\phi$\unboldmath}}
\def\pii{\mbox{\boldmath$\pi$\unboldmath}}
\def\Ph{\mbox{\boldmath$\Phi$\unboldmath}}
\def\Ps{\mbox{\boldmath$\Psi$\unboldmath}}
\def\tha{\mbox{\boldmath$\theta$\unboldmath}}
\def\Tha{\mbox{\boldmath$\Theta$\unboldmath}}
\def\muu{\mbox{\boldmath$\mu$\unboldmath}}
\def\Si{\mbox{\boldmath$\Sigma$\unboldmath}}
\def\si{\mbox{\boldmath$\sigma$\unboldmath}}
\def\Gam{\mbox{\boldmath$\Gamma$\unboldmath}}
\def\Lam{\mbox{\boldmath$\Lambda$\unboldmath}}
\def\De{\mbox{\boldmath$\Delta$\unboldmath}}
\def\Ome{\mbox{\boldmath$\Omega$\unboldmath}}
\def\TOme{\mbox{\boldmath$\hat{\Omega}$\unboldmath}}
\def\vps{\mbox{\boldmath$\varepsilon$\unboldmath}}
\newcommand{\ti}[1]{\tilde{#1}}
\def\Ncal{\mathcal{N}}
\def\argmax{\mathop{\rm argmax}}
\def\argmin{\mathop{\rm argmin}}
\providecommand{\abs}[1]{\lvert#1\rvert}
\providecommand{\norm}[2]{\lVert#1\rVert_{#2}}

\def\Zs{{\Z_{\mathrm{S}}}}
\def\Zl{{\Z_{\mathrm{L}}}}
\def\Yr{{\Y_{\mathrm{R}}}}
\def\Yg{{\Y_{\mathrm{G}}}}
\def\Yb{{\Y_{\mathrm{B}}}}
\def\Ar{{\A_{\mathrm{R}}}}
\def\Ag{{\A_{\mathrm{G}}}}
\def\Ab{{\A_{\mathrm{B}}}}
\def\As{{\A_{\mathrm{S}}}}
\def\Asr{{\A_{\mathrm{S}_{\mathrm{R}}}}}
\def\Asg{{\A_{\mathrm{S}_{\mathrm{G}}}}}
\def\Asb{{\A_{\mathrm{S}_{\mathrm{B}}}}}
\def\Or{{\Ome_{\mathrm{R}}}}
\def\Og{{\Ome_{\mathrm{G}}}}
\def\Ob{{\Ome_{\mathrm{B}}}}

\def\vec{\mathrm{vec}}
\def\fold{\mathrm{fold}}
\def\index{\mathrm{index}}
\def\sgn{\mathrm{sgn}}
\def\tr{\mathrm{tr}}
\def\rk{\mathrm{rank}}
\def\diag{\mathsf{diag}}
\def\const{\mathrm{Const}}
\def\dg{\mathsf{dg}}
\def\st{\mathsf{s.t.}}
\def\vect{\mathsf{vec}}
\def\MCAR{\mathrm{MCAR}}
\def\MSAR{\mathrm{MSAR}}
\def\etal{{\em et al.\/}\,}
\newcommand{\indep}{{\;\bot\!\!\!\!\!\!\bot\;}}

\def\Lsize{\hbox{\space \raise-2mm\hbox{$\textstyle \L \atop \scriptstyle {m\times 3n}$} \space}}
\def\Ssize{\hbox{\space \raise-2mm\hbox{$\textstyle \S \atop \scriptstyle {m\times 3n}$} \space}}
\def\Osize{\hbox{\space \raise-2mm\hbox{$\textstyle \Ome \atop \scriptstyle {m\times 3n}$} \space}}
\def\Tsize{\hbox{\space \raise-2mm\hbox{$\textstyle \T \atop \scriptstyle {3n\times n}$} \space}}
\def\Bsize{\hbox{\space \raise-2mm\hbox{$\textstyle \B \atop \scriptstyle {m\times n}$} \space}}

\newcommand{\twopartdef}[4]
{
	\left\{
		\begin{array}{ll}
			#1 & \mbox{if } #2 \\
			#3 & \mbox{if } #4
		\end{array}
	\right.
}

\newcommand{\tabincell}[2]{\begin{tabular}{@{}#1@{}}#2\end{tabular}}

\newtheorem{theorem}{Theorem}
\newtheorem{lemma}{Lemma}
\newtheorem{proposition}{Proposition}
\newtheorem{corollary}{Corollary}
\newtheorem{definition}{Definition}
\newtheorem{remark}{Remark}

\def\E{{\mathbb E}}
\def\R{{\mathbb R}}

\DeclarePairedDelimiter\ceil{\lceil}{\rceil}
\DeclarePairedDelimiter\floor{\lfloor}{\rfloor}

\newcommand{\ip}[2]{\left\langle #1, #2 \right \rangle}

\newtheorem{assumption}{Assumption}

\def\aggr{{\tt Aggr}}
\def\mean{{\tt Mean}}
\def\median{{\tt Median}}
\def\krum{{\tt Krum}}
\def\zeno{{\tt Zeno}}

\twocolumn[
\icmltitle{Zeno: Distributed Stochastic Gradient Descent with Suspicion-based Fault-tolerance}




\begin{icmlauthorlist}
\icmlauthor{Cong Xie}{UIUC}
\icmlauthor{Oluwasanmi Koyejo}{UIUC}
\icmlauthor{Indranil Gupta}{UIUC}
\end{icmlauthorlist}

\icmlaffiliation{UIUC}{Department of Computer Science, University of Illinois, Urbana-Champaign, USA}

\icmlcorrespondingauthor{Cong Xie}{cx2@illinois.edu}

\icmlkeywords{Robust, SGD}

\vskip 0.3in
]



\printAffiliationsAndNotice{}
%\icmlEqualContribution} % otherwise use the standard text.

\begin{abstract}
We present Zeno, a technique to make distributed machine learning, particularly Stochastic Gradient Descent (SGD), tolerant to an arbitrary number of faulty workers. Zeno generalizes previous results that assumed a majority of non-faulty nodes; we need assume only one non-faulty worker. Our key idea is to suspect workers that are potentially defective. Since this is likely to lead to false positives, we use a ranking-based preference mechanism. We prove the convergence of SGD for non-convex problems under these scenarios. Experimental results show that Zeno outperforms existing approaches.
\end{abstract}

\section{Introduction}

In distributed machine learning, one of the hardest problems today is fault-tolerance. Faulty workers may take arbitrary actions or modify their portion of the data and/or models arbitrarily. In addition to adversarial attacks on purpose, it is also common for the workers to have hardware or software failures, such as bit-flipping in the memory or communication media. While fault-tolerance has been studied for distributed machine learning~\cite{blanchard2017machine,Chen2017DistributedSM,yin2018byzantine,feng2014distributed,Su2016FaultTolerantMO,su2016defending,alistarh2018byzantine}, much of the work on fault-tolerant machine learning makes strong assumptions. For instance, a common assumption is that no more than 50\% of the workers are faulty ~\cite{blanchard2017machine,Chen2017DistributedSM,yin2018byzantine,Su2016FaultTolerantMO,alistarh2018byzantine}. 

We present Zeno, a new technique that generalizes the failure model so that we only require at least one non-faulty (good) worker. 
% We also allow faulty workers to collude. 
In particular, faulty gradients may pretend to be good by behaving similar to the correct gradients in variance and magnitude, making them hard to distinguish. It is also possible that in different iterations, different groups of workers are faulty, which means that we can not simply identify workers which are always faulty.

\begin{figure}[htb!]
\centering
\includegraphics[width=0.49\textwidth]{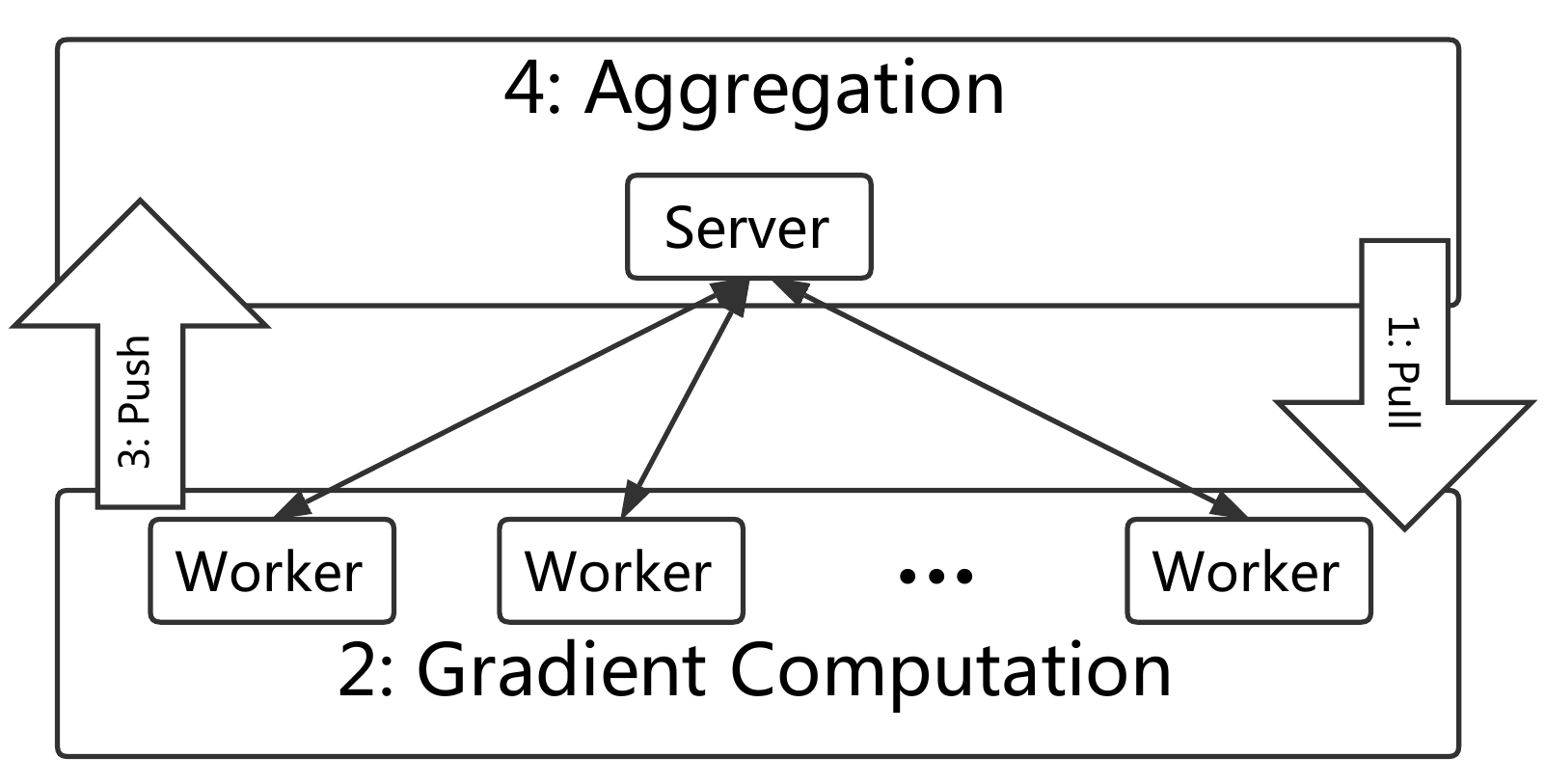}
\caption{Parameter Server architecture.}
\label{fig:ps}
\end{figure}

We focus on Stochastic Gradient Descent (SGD), and use the Parameter Server (PS) architecture~\cite{li2014scaling,li2014communication} for distributed SGD. As illustrated in Figure 1,  processes are composed of the server nodes and worker nodes. In each SGD iteration, the workers pull the latest model from the servers, estimate the gradients using the locally sampled training data, then push the gradient estimators to the servers. The servers aggregate the gradient estimators, and update the model by using the aggregated gradients.

Our approach, in a nutshell is the following. We treat each candidate gradient estimator as a suspect. We compute a score using a stochastic zero-order oracle. This ranking indicates how trustworthy the given worker is in that iteration. Then, we take the average over the several candidates with the highest scores. This allows us to tolerate a large number of incorrect gradients. We prove that the convergence is as fast as fault-free SGD. Further, the variance falls as the number of non-faulty workers increases. 

To the best of our knowledge this paper is the first to theoretically and empirically study cases where a majority of workers are faulty for non-convex problems. In summary, our contributions are:

\setitemize[0]{leftmargin=*}
\begin{itemize} 
\item A new approach for SGD with fault-tolerance, that works with an arbitrarily large number of faulty nodes as long as there is at least one non-faulty node.
\item Theoretically, the proposed algorithm converges as fast as distributed synchronous SGD without faulty workers, with the same asymptotic time complexity.
\item Experimental results validating that 1) existing majority-based robust algorithms may fail even when the number of faulty workers is lower than the majority, and 2) Zeno gracefully handles such cases.
\item The effectiveness of Zeno also extends to the case where the workers use disjoint local data to train the model, i.e., the local training data are not identically distributed across different workers. Theoretical and experimental analysis is also provided in this case. 
\end{itemize}

\section{Related Work}

Many approaches for improving failure tolerance are based on robust statistics. For instance, \citet{Chen2017DistributedSM,Su2016FaultTolerantMO,su2016defending} use geometric median as the aggregation rule. \citet{yin2018byzantine} establishes statistical error rates for marginal trimmed mean as the aggregation rule. Similar to these papers, our proposed algorithm also works under Byzantine settings.

There are also robust gradient aggregation rules that are not based on robust statistics. For example, \citet{blanchard2017machine} propose Krum, which select the candidates with minimal local sum of Euclidean distances. 
DRACO~\cite{chen2018draco} uses coding theory to ensure robustness.

\citet{alistarh2018byzantine} proposes a fault-tolerant SGD variant different from the robust aggregation rules. The algorithm utilizes historical information, and achieves the optimal sample complexity. 
% However, the algorithm requires the estimated upper bound of the variances of the stochastic gradients, which makes the algorithm less practical. Furthermore, there are no empirical results provided. 

Despite their differences, the existing majority-based methods for synchronous SGD~\cite{blanchard2017machine,Chen2017DistributedSM,yin2018byzantine,Su2016FaultTolerantMO,alistarh2018byzantine} 
assume that the non-faulty workers dominate the entire set of workers. Thus, such algorithms can trim the outliers from the candidates. However, in real-world failures or attacks, there are no guarantees that the number of faulty workers can be bounded from above. 
% In contrast, our suspicion-based aggregation rule, \textit{Zeno}, removes such constraint, and guarantees the convergence with provable rates. 

\section{Model}
\label{sec:model}
We consider the following optimization problem:
\begin{align*}
\min_{x \in \R^d} F(x), 
\end{align*}
where $F(x) = \E_{z \sim \mathcal{D}}[f(x; z)]$, $z$ is sampled from some unknown distribution $\mathcal{D}$, $d$ is the number of dimensions. We assume that there exists a minimizer of $F(x)$, which is denoted by $x^*$.

We solve this problem in a distributed manner with $m$ workers. In each iteration, each worker will sample $n$ independent and identically distributed~(i.i.d.) data points from the distribution $\mathcal{D}$, and compute the gradient of the local empirical loss $F_i(x) = \frac{1}{n} \sum_{j=1}^n f(x; z^{i,j}), \forall i \in [m]$, where $z^{i,j}$ is the $j$th sampled data on the $i$th worker. The servers will collect and aggregate the gradients sent by the works, and update the model as follows:
\begin{align*}
x^{t+1} = x^t - \gamma^t \aggr(\{g_i(x^t): i \in [m]\}),
\end{align*}
where $\aggr(\cdot)$ is an aggregation rule (e.g., averaging), and 
\begin{align}
\label{equ:byz_grad}
g_i(x^t) = 
\begin{cases}
* & \mbox{$i$th worker is faulty}, \\
\nabla F_i(x^t) & \mbox{otherwise,}
\end{cases}
\end{align}
where ``$*$" represents arbitrary values.

Formally, we define the failure model in synchronous SGD as follows.
\begin{definition} 
\label{def:byz}
(Failure Model). In the $t^{\mbox{th}}$ iteration, let $\{v_i^t: i \in [m]\}$ be i.i.d. random vectors in $\R^d$, where $v_i^t = \nabla F_i(x^t)$. The set of correct vectors $\{v_i^t: i \in [m]\}$ is partially replaced by faulty vectors, which results in $\{\tilde{v}_i^t: i \in [m]\}$, where $\tilde{v}_i^t = g_i(x^t)$ as defined in Equation~(\ref{equ:byz_grad}). In other words, a correct/non-faulty gradient is $\nabla F_i(x^t)$, while a faulty gradient, marked as ``$*$", is assigned arbitrary value.  We assume that $q$ out of $m$ vectors are faulty, where $q < m$. Furthermore, the indices of faulty workers can change across different iterations. 
\end{definition}
We observe that in the worst case, the failure model in Definition~\ref{def:byz} is equivalent to the Byzantine failures introduced in \citet{blanchard2017machine,Chen2017DistributedSM,yin2018byzantine}. In particular, if the failures are caused by attackers, the failure model includes the case where the attackers collude.

To help understand the failure model in synchronous SGD, we illustrate a toy example in Figure~\ref{fig:byz_example}.

The notations used in this paper is summarized in Table~\ref{tbl:notations}.
\begin{figure}[htb!]
\centering
\includegraphics[width=0.45\textwidth]{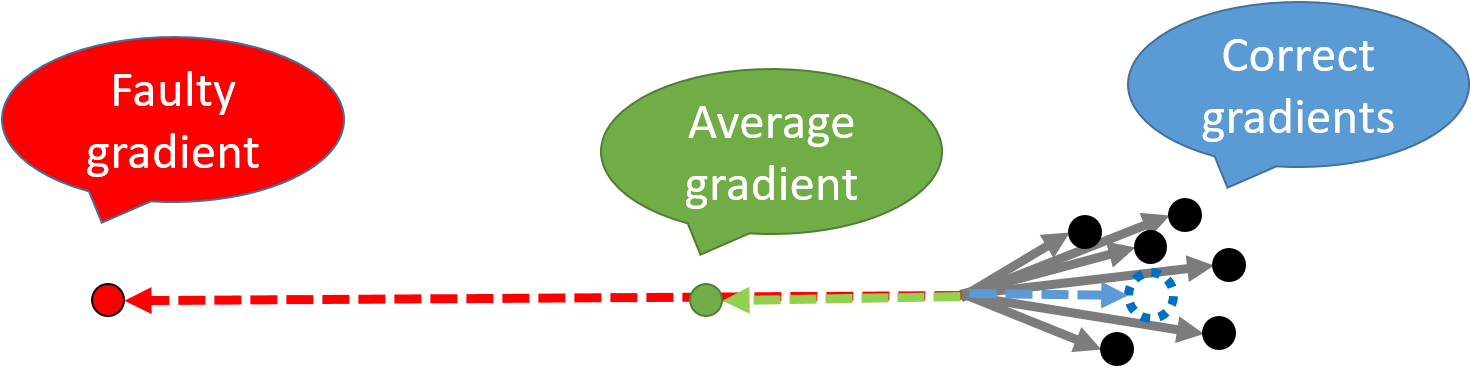}
\caption{A toy example of the failure model in synchronous SGD. There are $m=7$ candidate gradient estimators. The black dots represent the correct gradients, where $\tilde{v}_i = \nabla F_i(x^t)$, $i \in [m-1]$. The red dot represents the faulty gradient, whose value (in the worst case) is $\tilde{v}_m = \epsilon \nabla F_m(x^t)$, where $\epsilon<0$ is a large negative constant. The blue dashed circle represent the expectation of the true gradient $\nabla F(x^t)$. Thus, the averaged gradient, which will be computed by the server, represented by the green dot, is far away from the true gradient, which is harmful to the model training.}
\label{fig:byz_example}
\end{figure}
\vspace{-0.3cm}
\begin{table}[htb]
\caption{Notations}
\label{tbl:notations}
\begin{center}
\begin{small}
\begin{tabular}{|l|l|}
\hline 
Notation  & Description \\ \hline
$m$    & Number of workers \\ \hline
$n$    & Number of samples on each worker \\ \hline
$T$    & Number of epochs  \\ \hline
$[m]$    & Set of integers $\{1, \ldots, m \}$  \\ \hline
$q$    & Number of faulty workers  \\ \hline
$b$    & Trim parameter of \texttt{Zeno}  \\ \hline
$\gamma$    & Learning rate  \\ \hline
$\rho$    & Regularization weight of \texttt{Zeno}  \\ \hline
$n_r$    & Batch size of \texttt{Zeno}   \\ \hline
$\| \cdot \|$    & All the norms in this paper are $l_2$-norms  \\ \hline
\end{tabular}
\end{small}
\end{center}
\vspace{-0.3cm}
\end{table}

\section{Methodology}
In contrast to the existing majority-based methods, we compute a score for each candidate gradient estimator by using the stochastic zero-order oracle. We rank each candidate gradient estimator based on the estimated descent of the loss function, and the magnitudes. Then, the algorithm aggregates the candidates with highest scores. The score roughly indicates how trustworthy each candidate is.
\begin{definition}(Stochastic Descendant Score)
\label{def:score}
Denote $f_r(x) = \frac{1}{n_r} \sum_{i=1}^{n_r} f(x; z_i)$, where $z_i$'s are i.i.d. samples drawn from $\mathcal{D}$, and $n_r$ is the batch size of $f_r(\cdot)$. $\E[f_r(x)] = F(x)$. For any update~(gradient estimator) $u$, based on the current parameter $x$, learning rate $\gamma$, and a constant weight $\rho > 0$, we define its \textit{stochastic descendant score} as follows:
\begin{align*}
Score_{\gamma, \rho}(u, x) = f_r(x) - f_r(x - \gamma u) - \rho \| u \|^2.
\end{align*} 
\end{definition}
The score defined in Definition~\ref{def:score} is composed of two parts: the estimated descendant of the loss function, and the magnitude of the update. The score increases when the estimated descendant of the loss function, $f_r(x) - f_r(x - \gamma \tilde{v}_i)$, increases. The score decreases when the magnitude of the update, $\| \tilde{v}_i \|^2$, increases. Intuitively, the larger descendant suggests faster convergence, and the smaller magnitude suggests a smaller change. Even if a gradient is faulty, a smaller change makes it less harmful and easier to be cancelled by the correct gradients.

Using the score defined above, we establish the following suspicion-based aggregation rule. We ignore the index of iterations, $t$, for convenience.

\begin{algorithm}[htb!]
\caption{Zeno}
\begin{algorithmic}
\vspace{0.2cm}
\STATE{\large\underline{\textbf{Server}}} 
\STATE Input: $\rho$~(defined in Definition~\ref{def:score}), $b$~(defined in Definition~\ref{def:zeno})
\STATE $x^0 \leftarrow rand()$ \COMMENT{Initialization}
\FOR{$t = 1, \ldots, T$}
	\STATE Broadcast $x^{t-1}$ to all the workers
	\STATE Wait until all the gradients $\{\tilde{v}_i^t: i \in [m]\}$ arrive
	\STATE Draw the samples for evaluating stochastic descendant score $f_r^t(\cdot)$ as defined in Definition~\ref{def:score}
	\STATE Compute $\bar{\tilde{v}}^t = \zeno_b(\{\tilde{v}_i^t: i \in [m]\})$ as defined in Definition~\ref{def:zeno}
	\STATE Update the parameter $x^t \leftarrow x^{t-1} - \gamma^t \bar{\tilde{v}}^t$
\ENDFOR
\end{algorithmic}
\begin{algorithmic}
\vspace{0.2cm}
\STATE{\large\underline{\textbf{Worker}} $i = 1, \ldots, m$} 
\FOR{$t = 1, \ldots, T$}
	\STATE Receive $x^{t-1}$ from the server
	\STATE Draw the samples, compute, and send the gradient $v_i^t = \nabla F_i^t(x^{t-1})$ to the server
\ENDFOR
\end{algorithmic}
\label{alg:zeno}
\end{algorithm}

\begin{definition}(Suspicion-based Aggregation)
\label{def:zeno}
Assume that among the gradient estimators $\{\tilde{v}_i: i \in [m]\}$, $q$ elements are faulty, and $x$ is the current value of the parameters. We sort the sequence by the stochastic descendant score defined in Definition~\ref{def:score}, which results in $\{\tilde{v}_{(i)}: i \in [m]\}$, where 
\begin{align*}
Score_{\gamma, \rho}(\tilde{v}_{(1)}, x) \geq \ldots \geq Score_{\gamma, \rho}(\tilde{v}_{(m)}, x).
\end{align*}
In other words, $\tilde{v}_{(i)}$ is the vector with the $i$th highest score in $\{\tilde{v}_i: i \in [m]\}$.

The proposed aggregation rule, \textit{Zeno}, aggregates the gradient estimators by taking the average of the first $m-b$ elements in $\{\tilde{v}_{(i)}: i \in [m]\}$~(the gradient estimators with the $(m-b)$ highest scores), where $m > b \geq q$: 
\begin{align*}
\zeno_b(\{\tilde{v}_i: i \in [m]\}) = \frac{1}{m-b} \sum_{i=1}^{m-b} \tilde{v}_{(i)}.
\end{align*}
\end{definition}

Note that $z_i$'s (in Definition~\ref{def:score}) are independently sampled in different iterations. Furthermore, in each iteration, $z_i$'s are sampled after the arrival of the candidate gradient estimators $\tilde{v}_i^t$ on the server. Since the faulty workers are not predictive, they cannot obtain the exact information of $f_r(\cdot)$, which means that the faulty gradients are independent of $f_r(\cdot)$, though the faulty workers can know $\E[f_r(\cdot)]$.

Using \texttt{Zeno} as the aggregation rule, the detailed distributed synchronous SGD is shown in Algorithm~\ref{alg:zeno}.

\begin{figure}[htb!]
\centering
\includegraphics[width=0.45\textwidth,height=4cm]{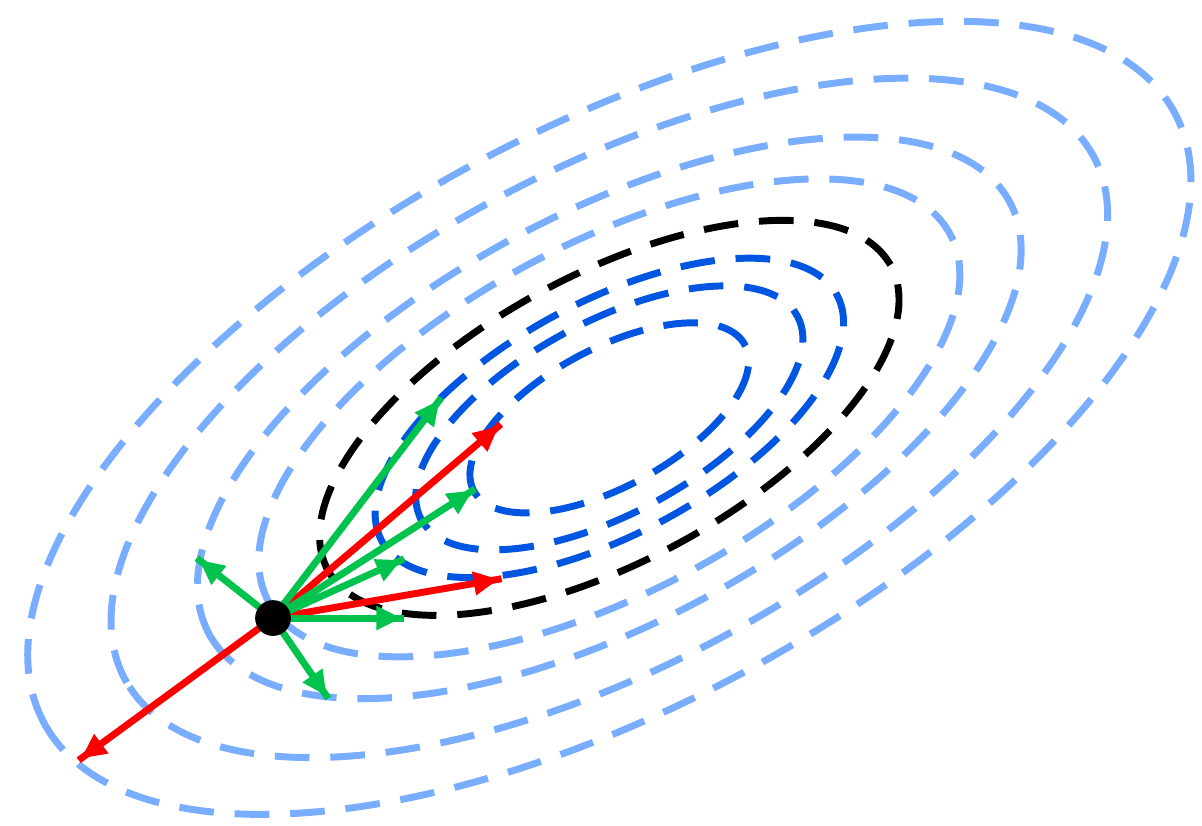}
\caption{\texttt{Zeno} on loss surface contours. We use the notations in Definition~\ref{def:score} and \ref{def:zeno}. The black dot is the current parameter $x$. The arrows are the candidate updates $\{\tilde{v}_i: i \in [m]\}$. Red arrows are the incorrect updates. Green arrows are the correct updates. Taking $b=3$, \texttt{Zeno} filters out the 3 arrows pointing outside the black dashed circle. These 3 updates have the least descendant of the loss function, among all the updates. There are some incorrect updates (the red arrow) remaining inside the boundary. However, since they are bounded by the correct updates, the remaining incorrect updates are harmless.}
\label{fig:contour}
\end{figure}

In Figure~\ref{fig:contour}, we visualize the intuition underlying \texttt{Zeno}. It is illustrated that all the selected candidates (arrows pointing inside the black dashed circle) are bounded by at least one honest candidate. In other words, \texttt{Zeno} uses at least one honest candidate to establish a boundary (the black dashed circle), which filter out the potentially harmful candidates. The candidates inside the boundary are harmless, no matter they are actually faulty or not.

\section{Theoretical Guarantees}

In this section, we prove the convergence of synchronous SGD with \texttt{Zeno} as the aggregation rule under our failure model.
We start with the assumptions required by the convergence guarantees. The two basic assumptions are the smoothness of the loss function, and the bounded variance of the (non-faulty) gradient estimators. 

\subsection{Assumptions} 

In this section, we highlight the necessary assumption for stochastic descendant score, followed by the assumptions for convergence guarantees. 

\begin{assumption} (Unbiased evaluation)
\label{asm:score}
We assume that the stochastic loss function, $f_r(x)$, evaluated in the stochastic descendant score in Definition~\ref{def:score}, is an unbiased estimator of the global loss function $F(x)$, i.e., $\E[f_r(x)] = F(x)$.
\end{assumption}

% Note that we do not make any assumption for the \texttt{Zeno} batch size $n_r$ or the variance of $f_r(x)$.

\begin{assumption} (Bounded Taylor's Approximation)
\label{asm:smooth_cvx}
We assume that $f(x;z)$ has $L$-smoothness and $\mu$-lower-bounded Taylor's approximation~(also called $\mu$-weak convexity):
$
\ip{\nabla f(x; z)}{y-x} + \frac{\mu}{2} \|y-x\|^2 \leq f(y;z) - f(x;z) 
\leq \ip{\nabla f(x; z)}{y-x} + \frac{L}{2} \|y-x\|^2,
$
where $\mu \leq L$, and $L > 0$.
\end{assumption}
Note that Assumption~\ref{asm:smooth_cvx} covers the case of non-convexity by taking $\mu < 0$, non-strong convexity by taking $\mu = 0$, and strong convexity by taking $\mu > 0$.

\begin{assumption} (Bounded Variance)
\label{asm:variance}
We assume that in any iteration, any correct gradient estimator $v_i = \nabla F_i(x)$  has the upper-bounded variance:
$
\E \left\|v_i - \E\left[ v_i \right] \right\|^2 \leq V.
$
Furthermore, we assume that 
$
\E\|v_i\|^2 \leq G.
$
\end{assumption}

In general, Assumption~\ref{asm:variance} bounds the variance and the second moment of the correct gradients of any sample loss function $f(x;z)$, $\forall z \sim \mathcal{D}$.

\begin{remark}
Note that for the faulty gradients in our failure model, none of the assumptions above holds.
\end{remark}

\subsection{Convergence Guarantees}

For general functions, including convex and non-convex functions, we provide the following convergence guarantee. The proof can be found in the appendix.
\begin{theorem}
\label{thm:byz_step}
For $\forall x \in \R^d$, denote
\begin{align*}
\tilde{v}_i = 
\begin{cases}
* & \mbox{$i$th worker is faulty}, \\
\nabla F_i(x) & \mbox{otherwise,}
\end{cases}
\end{align*}
where $i \in [m]$, with $\E[ \nabla F_i(x) ] = \nabla F(x)$, and $\bar{\tilde{v}} = \zeno_b(\{\tilde{v}_i: i \in [m]\})$.
Taking $\gamma \leq \frac{1}{L}$, $\rho = \frac{\beta \gamma^2}{2}$, and $\beta > \max(0, -\mu)$, we have 
\begin{align*}
&\E \left[ F(x - \gamma \bar{\tilde{v}}) \right] - F(x) \leq -\frac{\gamma}{2} \| \nabla F(x) \|^2 \\
&\quad + \frac{\gamma(b-q + 1)(m-q)V}{(m-b)^2} + \frac{(L+\beta)\gamma^2 G}{2}.
\end{align*}
\end{theorem}

\begin{corollary}
\label{cor:fixed_lr}
Take $\gamma = \frac{1}{L \sqrt{T}}$, $\rho = \frac{\beta\gamma^2}{2}$, and $\beta > \max(0, -\mu)$. Using $\zeno$, with  $\E[ \nabla F_i(x^t) ] = \nabla F(x^t)$ for $\forall t \in \{0, \ldots, T\}$, after $T$ iterations, we have
\begin{align*}
&\frac{\sum_{t=0}^{T-1} \E \| \nabla F(x^t) \|^2}{T} \\
&\leq  \mathcal{O}\left( \frac{1}{\sqrt{T}} \right) + \mathcal{O} \left( \frac{(b-q + 1)(m-q)}{(m-b)^2} \right).
\end{align*}
\end{corollary}

Now, we consider a more general case, where each worker has a disjoint (non-identically distributed) local dataset for training, which results in non-identically distributed gradient estimators. The server is still aware of the the entire dataset. For example, in volunteer computing~\cite{Meeds2015MLitBML,Miura2015ImplementationOA}, the server/coordinator can assign disjoint tasks/subsets of training data to the workers, while the server holds the entire training dataset. In this scenario, we have the following convergence guarantee.
\begin{corollary}
\label{cor:disjoint}
Assume that 
\begin{align*}
F(x) = \frac{1}{m} \sum_{i \in [m]} \E \left[ F_i(x) \right], \E \left[ F_i(x) \right] \neq \E \left[ F_j(x) \right],
\end{align*}
for $\forall i, j \in [m]$, $i \neq j$. For the stochastic descendant score, we have $\E \left[ f_r(x) \right] = F(x)$. Assumption~\ref{asm:score}, \ref{asm:smooth_cvx}, and \ref{asm:variance} hold.
Take $\gamma = \frac{1}{L \sqrt{T}}$, $\rho = \frac{\beta\gamma^2}{2}$, and $\beta > \max(0, -\mu)$. Using $\zeno$, after $T$ iterations, we have
\begin{align*}
&\frac{\sum_{t=0}^{T-1} \E \| \nabla F(x^t) \|^2}{T} \\
&\leq \mathcal{O}\left( \frac{1}{\sqrt{T}} \right) + \mathcal{O}\left(\frac{b}{m}\right) + \mathcal{O}\left(\frac{b^2(m-q)}{m^2(m-b)}\right).
\end{align*}
\end{corollary}

These two corollaries tell us that when using \texttt{Zeno} as the aggregation rule, even if there are failures, the convergence rate can be as fast as fault-free distributed synchronous SGD. The variance decreases when the number of workers $m$ increases, or the estimated number of faulty workers $b$ decreases.

\begin{remark}
There are two practical concerns for the proposed algorithm. 
First, by increasing the batch size of $f_r(\cdot)$~($n_r$ in Definition~\ref{def:score}), the stochastic descendant score will be potentially more stable. However, according to Theorem~\ref{thm:byz_step} and Corollary~\ref{cor:fixed_lr} and \ref{cor:disjoint}, the convergence rate is independent of the variance of $f_r$. Thus, theoretically we can use a single sample to evaluate the stochastic descendant score. 
Second, theoretically we need larger $\rho$ for non-convex problems. However, larger $\rho$ makes \texttt{Zeno} less sensitive to the descendant of the loss function, which potentially increases the risk of aggregating harmful candidates. In practice, we can use a small $\rho$ by assuming the local convexity of the loss functions. 
\end{remark}

\subsection{Implementation Details: Time Complexity}
Unlike the majority-based aggregation rules, the time complexity of \texttt{Zeno} is not trivial to analyze. Note that the convergence rate is independent of the variance of $f_r$, which means that we can use a single sample~($n_r = 1$) to evaluate $f_r$ to achieve the same convergence rate. Furthermore, in general, when evaluating the loss function on a single sample, the time complexity is roughly linear to the number of parameters $d$. Thus, informally, the time complexity of \texttt{Zeno} is $\mathcal{O}(dm)$ for one iteration, which is the same as \texttt{Mean} and \texttt{Median} aggregation rules. For comparison, note that the time complexity of \texttt{Krum} is $\mathcal{O}(dm^2)$.

\begin{figure*}[htb!]
\centering
\subfigure[Top-1 accuracy on testing set, with $q=8$]{\includegraphics[width=0.49\textwidth,height=3.9cm]{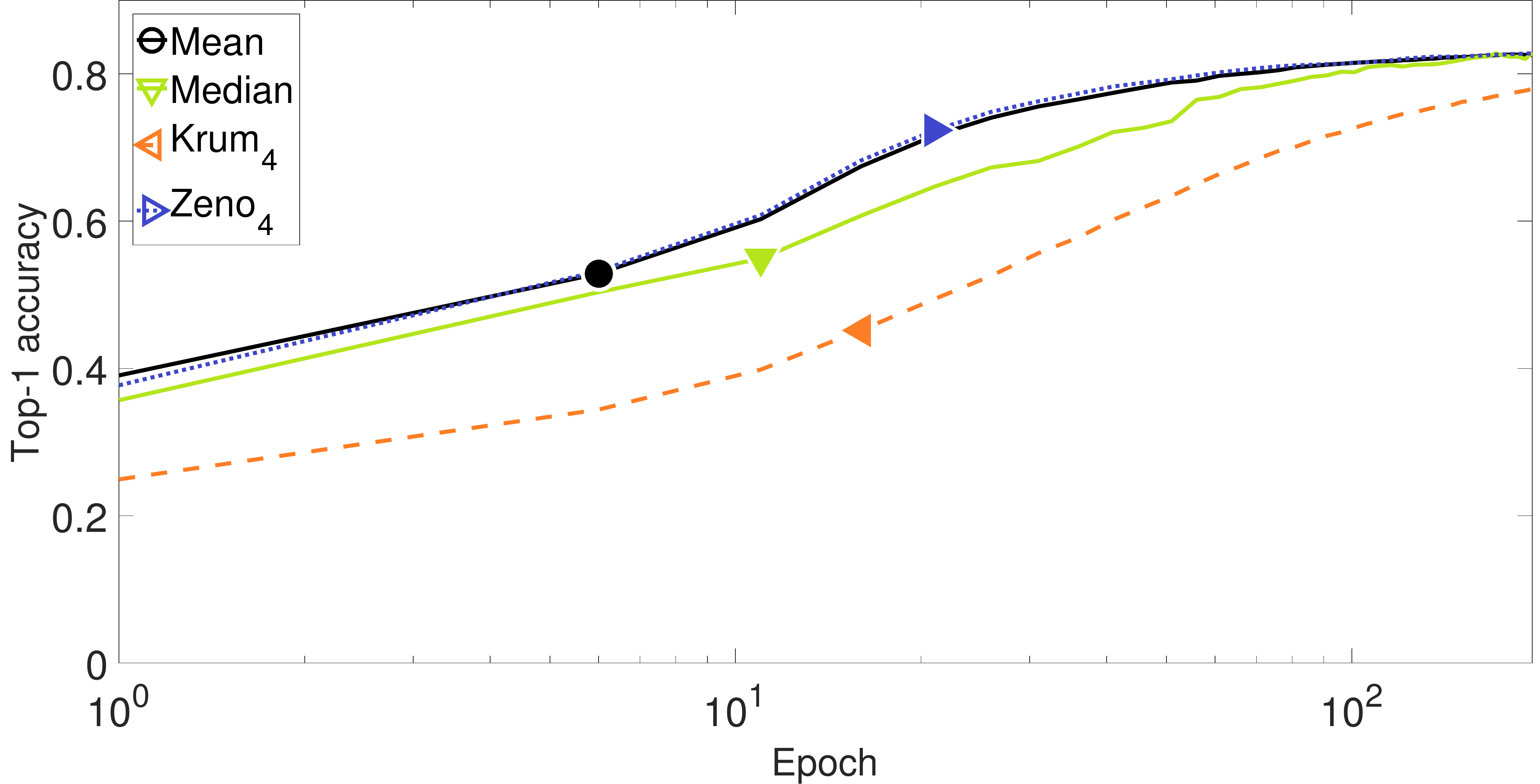}}
\subfigure[Cross entropy on training set, with $q=8$]{\includegraphics[width=0.49\textwidth,height=3.9cm]{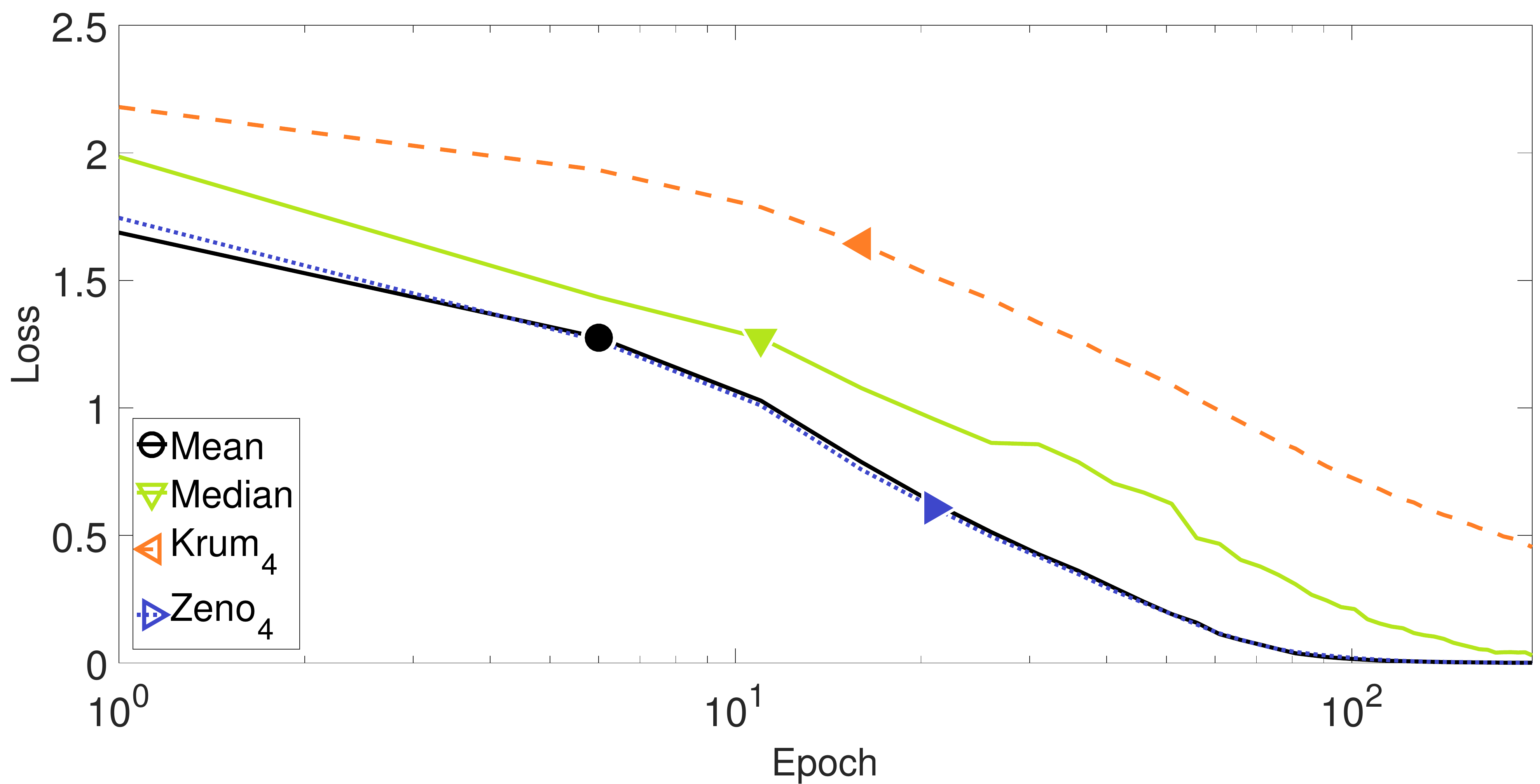}}
\caption{Convergence on i.i.d. training data, without failures. Batch size on the workers is $100$. Batch size of \texttt{Zeno} is $n_r=4$. $\rho=0.0005$. $\gamma = 0.1$.  Each epoch has 25 iterations. \texttt{Zeno} performs similar to \texttt{Mean}.}
\label{fig:iid_nobyz}
\end{figure*}
\begin{figure*}[htb!]
\centering
\subfigure[Top-1 accuracy on testing set, with $q=8$]{\includegraphics[width=0.49\textwidth,height=3.9cm]{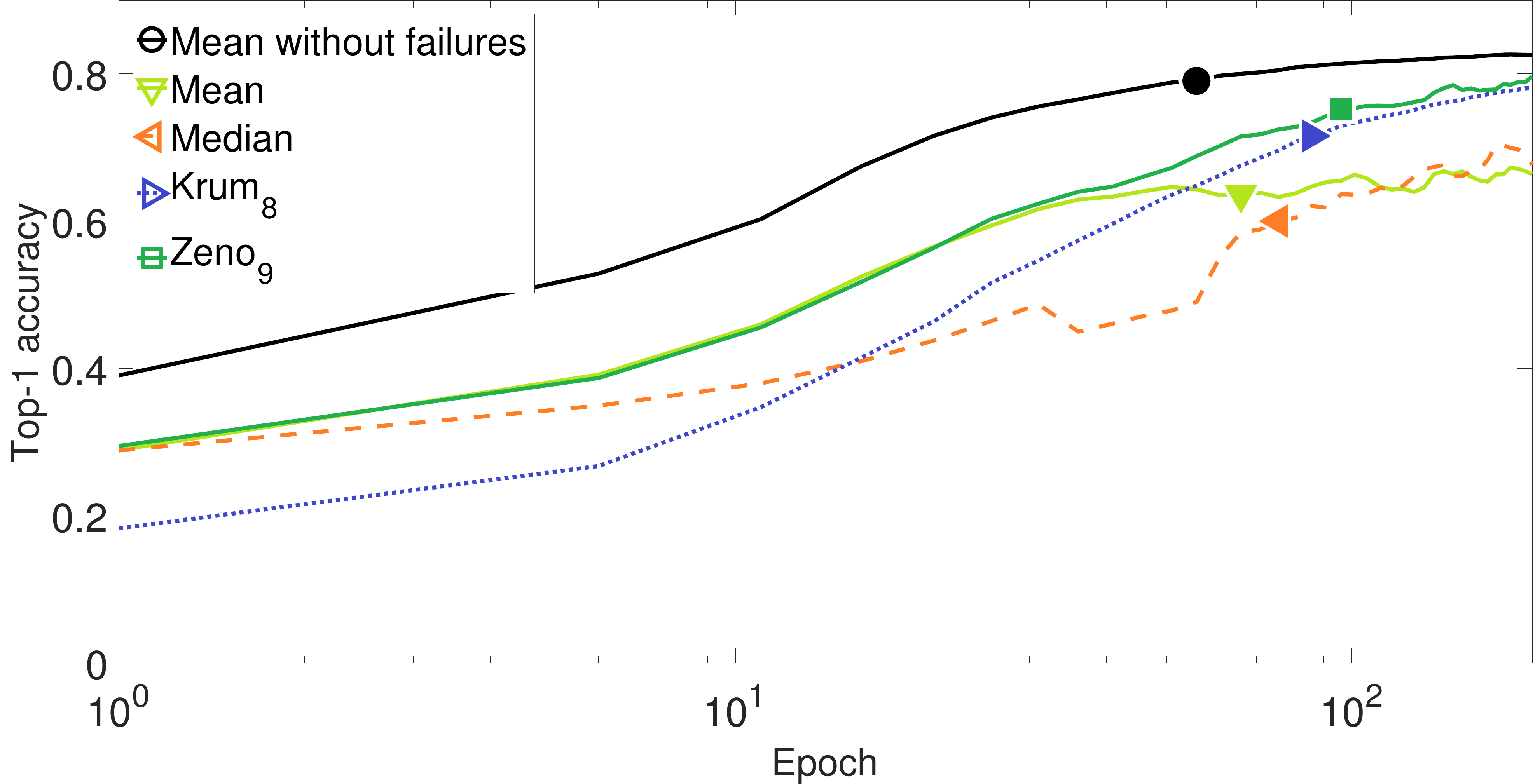}}
\subfigure[Cross entropy on training set, with $q=8$]{\includegraphics[width=0.49\textwidth,height=3.9cm]{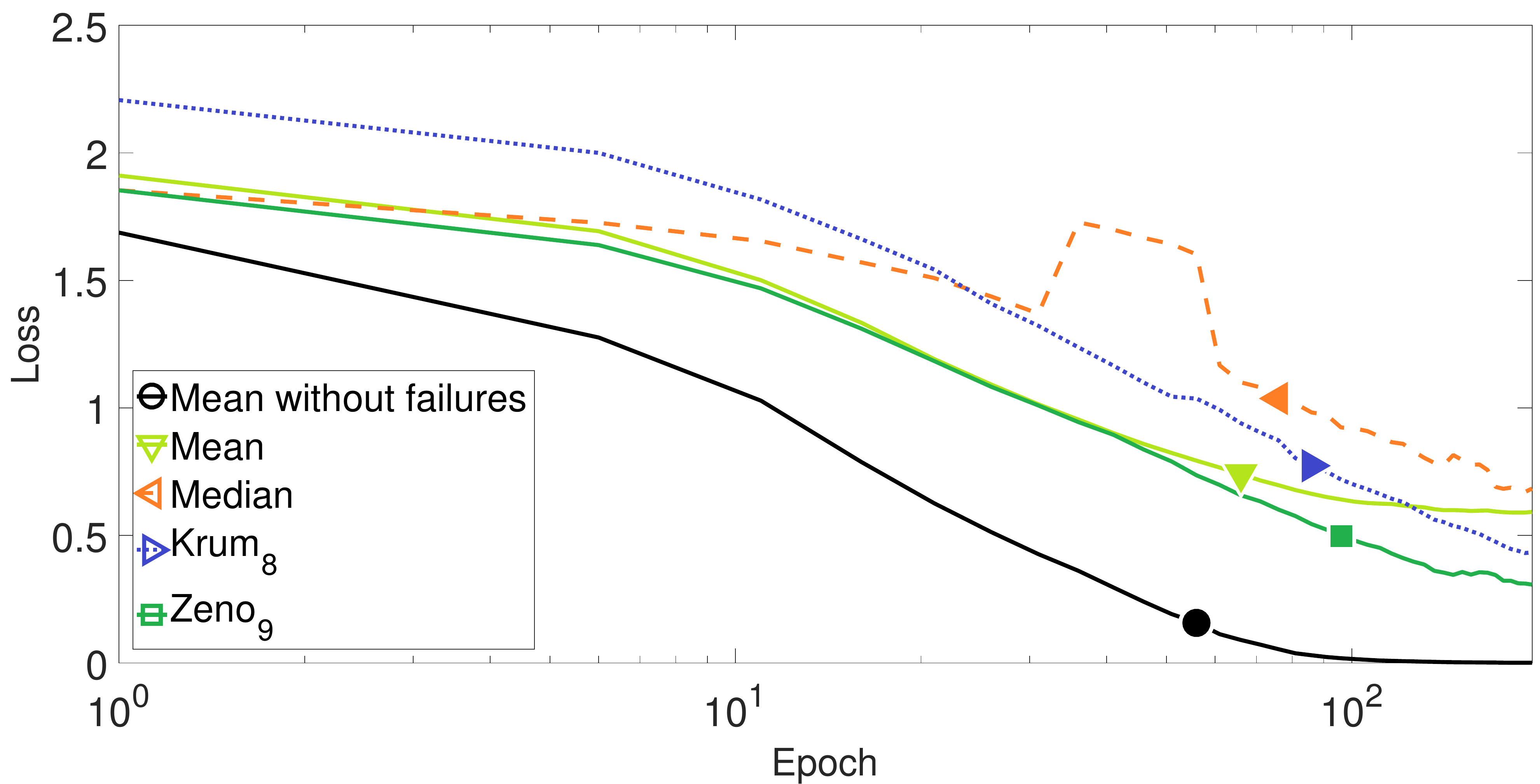}}
\subfigure[Top-1 accuracy on testing set, with $q=12$]{\includegraphics[width=0.49\textwidth,height=3.9cm]{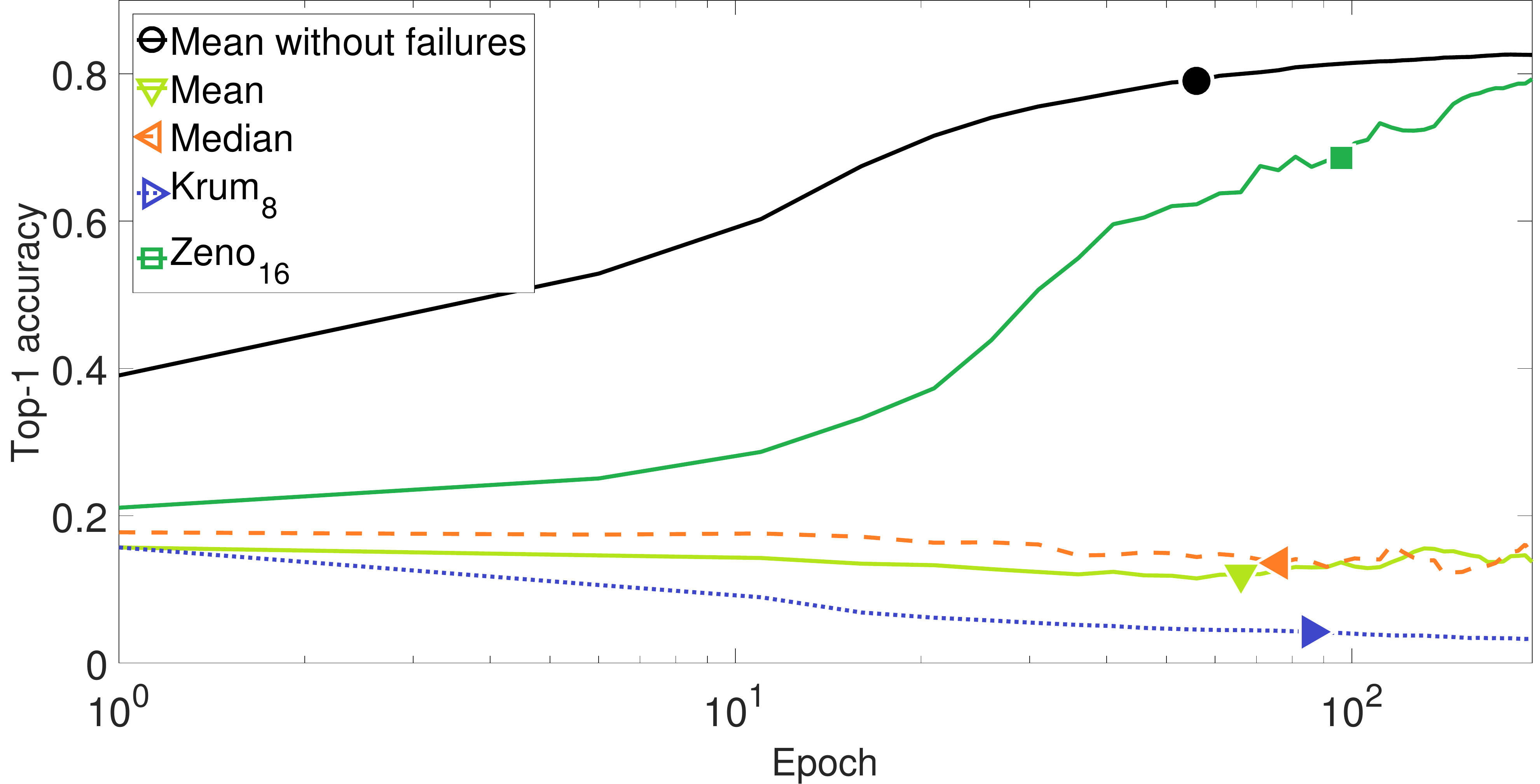}}
\subfigure[Cross entropy on training set, with $q=12$]{\includegraphics[width=0.49\textwidth,height=3.9cm]{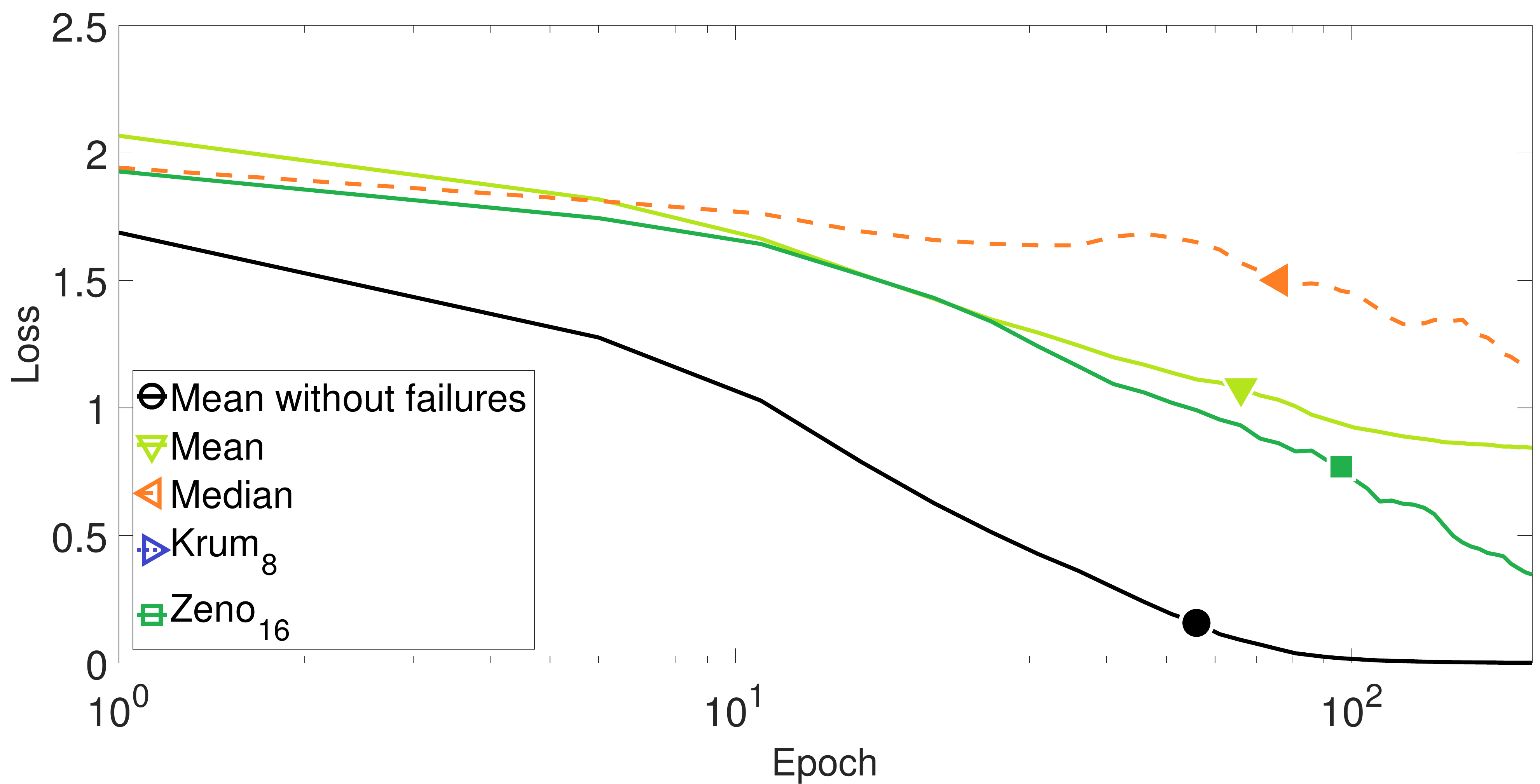}}
\caption{Convergence on i.i.d. training data, with label-flipping failures. Batch size on the workers is $100$. Batch size of \texttt{Zeno} is $n_r=4$. $\rho=0.0005$. $\gamma = 0.1$.  Each epoch has 25 iterations. \texttt{Zeno} outperforms all the baselines, especially when $q=12$.}
\label{fig:iid_labelflip}
\end{figure*}
\begin{figure*}[htb!]
\centering
\subfigure[Top-1 accuracy on testing set, with $q=8$]{\includegraphics[width=0.49\textwidth,height=3.9cm]{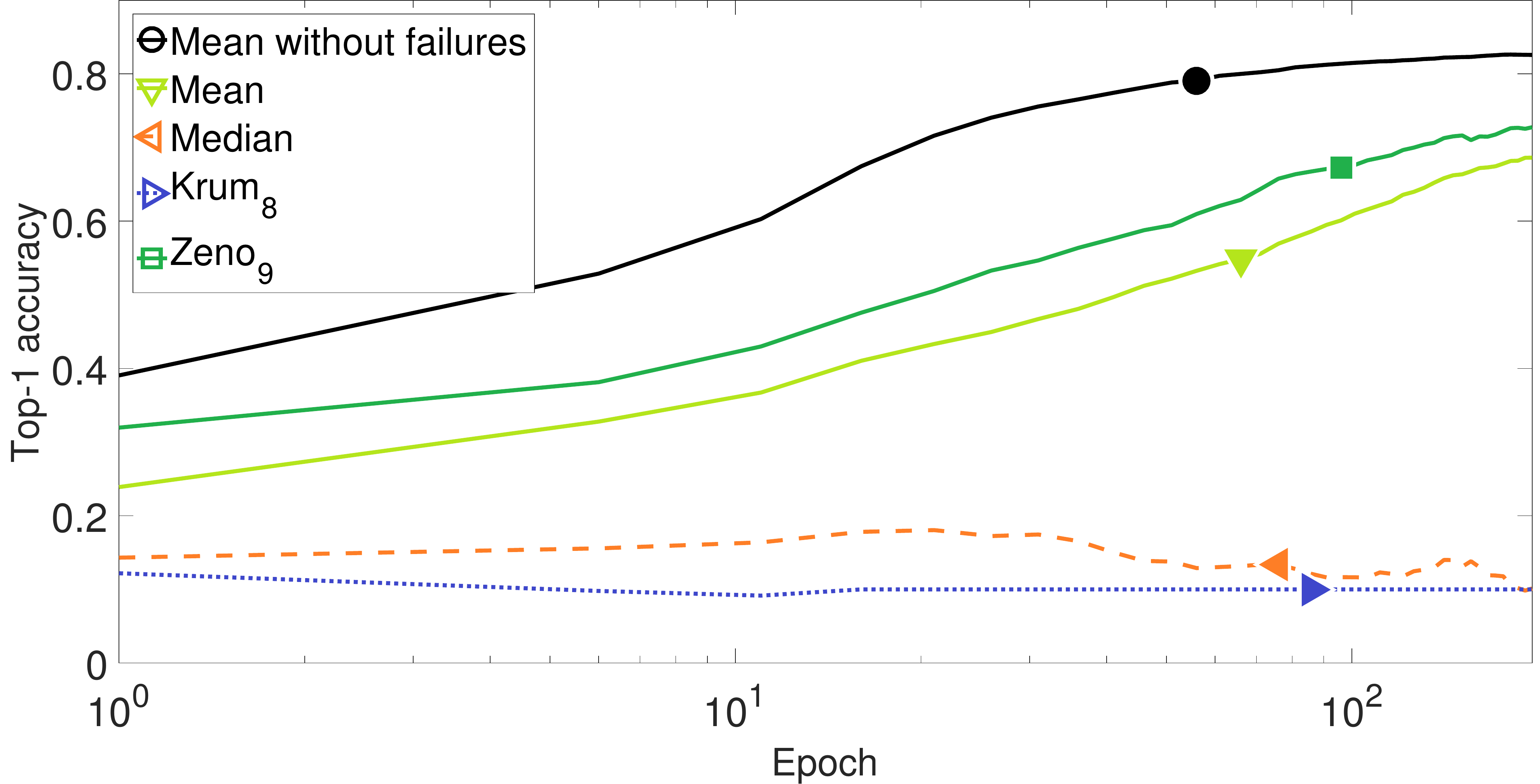}}
\subfigure[Cross entropy on training set, with $q=8$]{\includegraphics[width=0.49\textwidth,height=3.9cm]{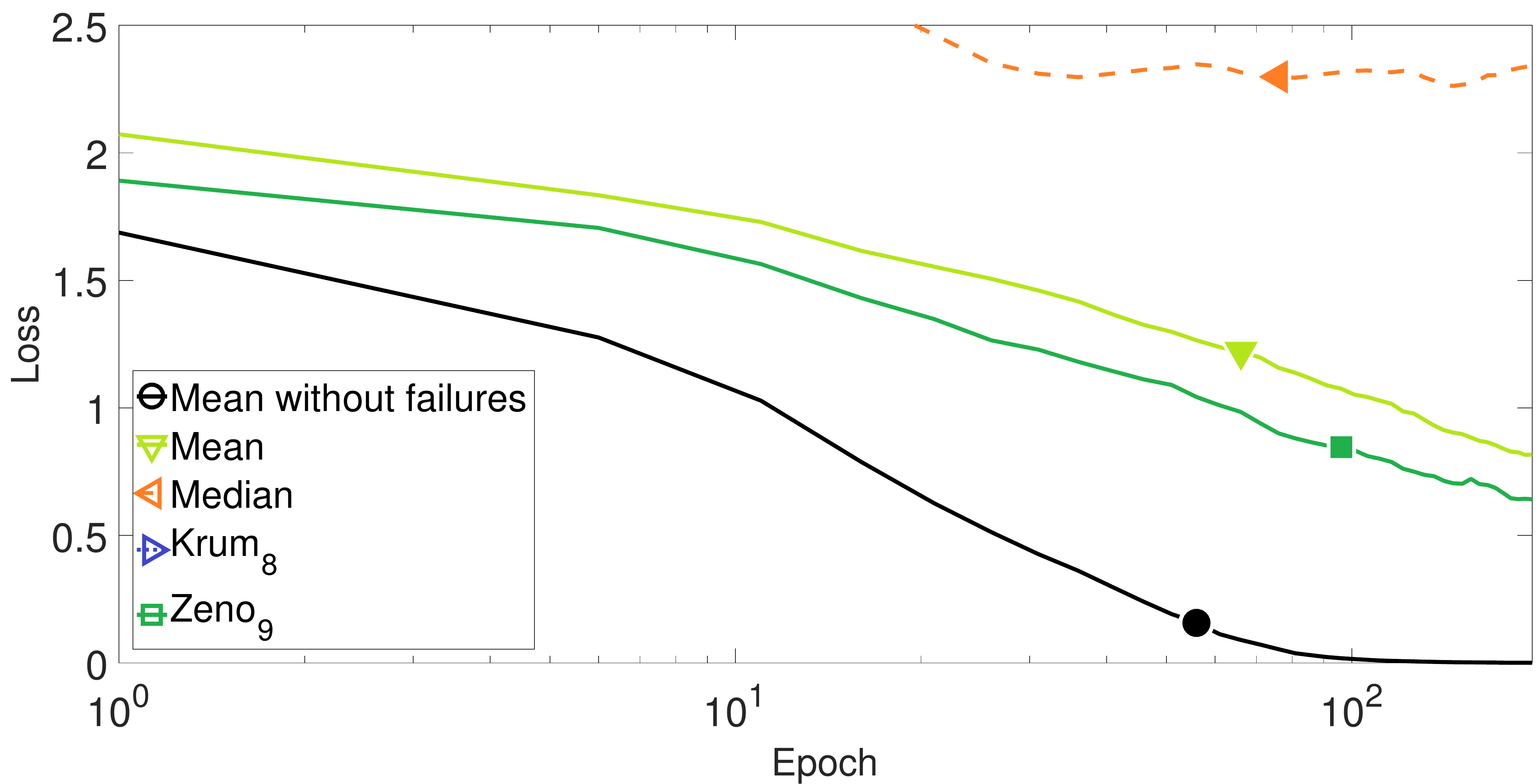}}
\subfigure[Top-1 accuracy on testing set, with $q=12$]{\includegraphics[width=0.49\textwidth,height=3.9cm]{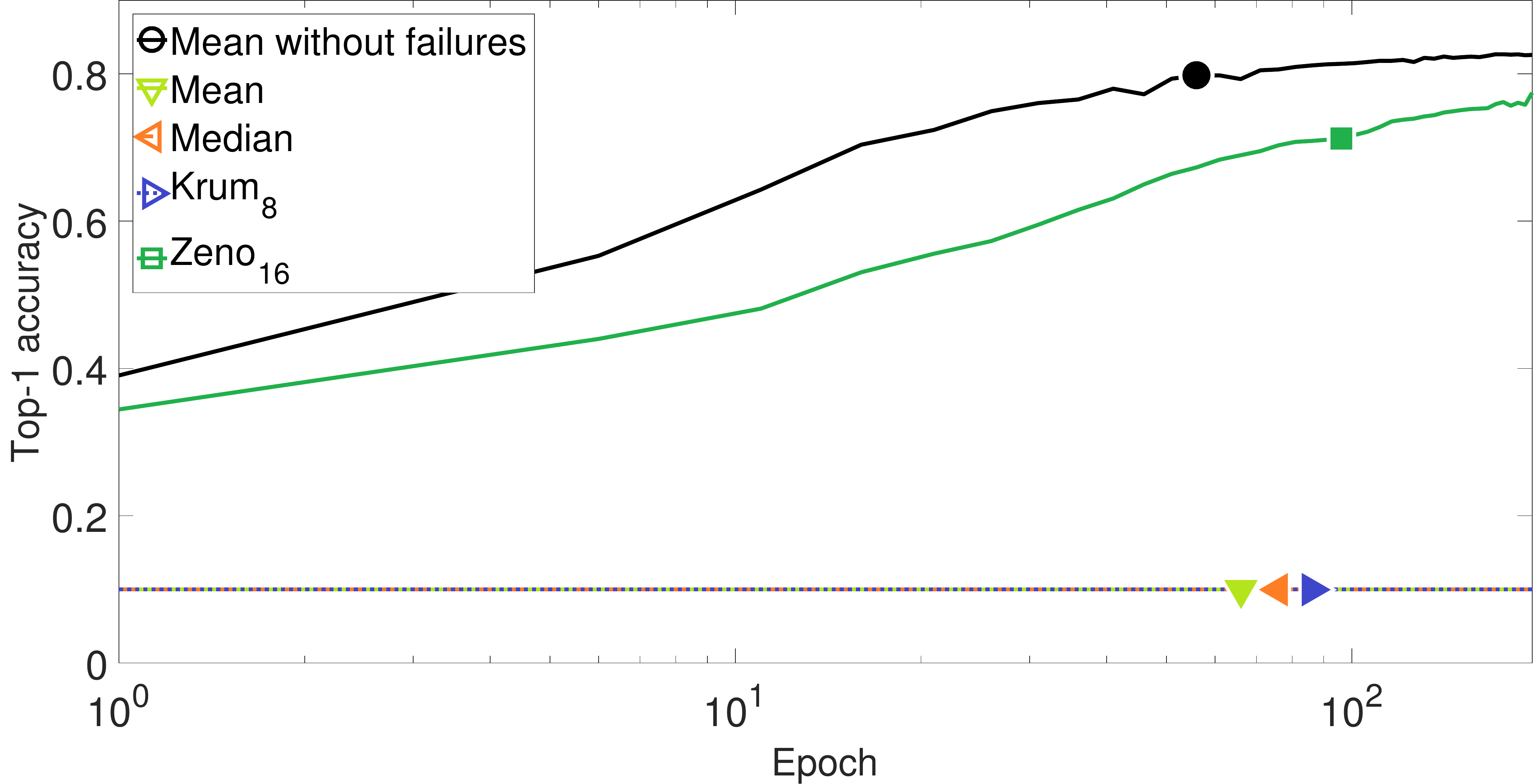}}
\subfigure[Cross entropy on training set, with $q=12$]{\includegraphics[width=0.49\textwidth,height=3.9cm]{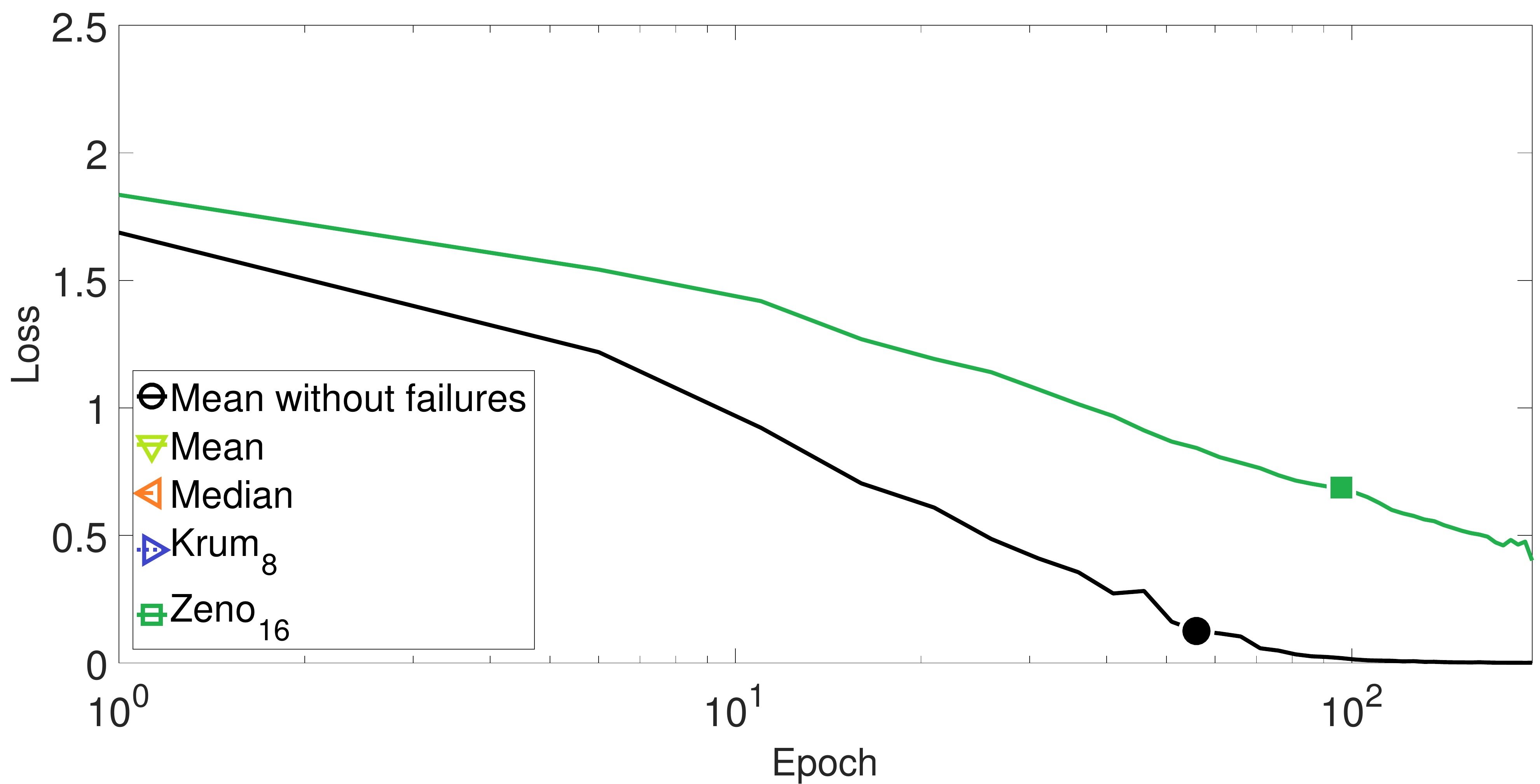}}
\caption{Convergence on i.i.d. training data, with bit-flipping failures. Batch size on the workers is $100$. Batch size of \texttt{Zeno} is $n_r=4$. $\rho=0.0005$. $\gamma = 0.1$.  Each epoch has 25 iterations. \texttt{Zeno} outperforms all the baselines, especially when $q=12$.}
\label{fig:iid_signflip}
\end{figure*}
\begin{figure*}[htb!]
\centering
\subfigure[Top-1 accuracy on testing set, with $q=8$]{\includegraphics[width=0.49\textwidth,height=3.9cm]{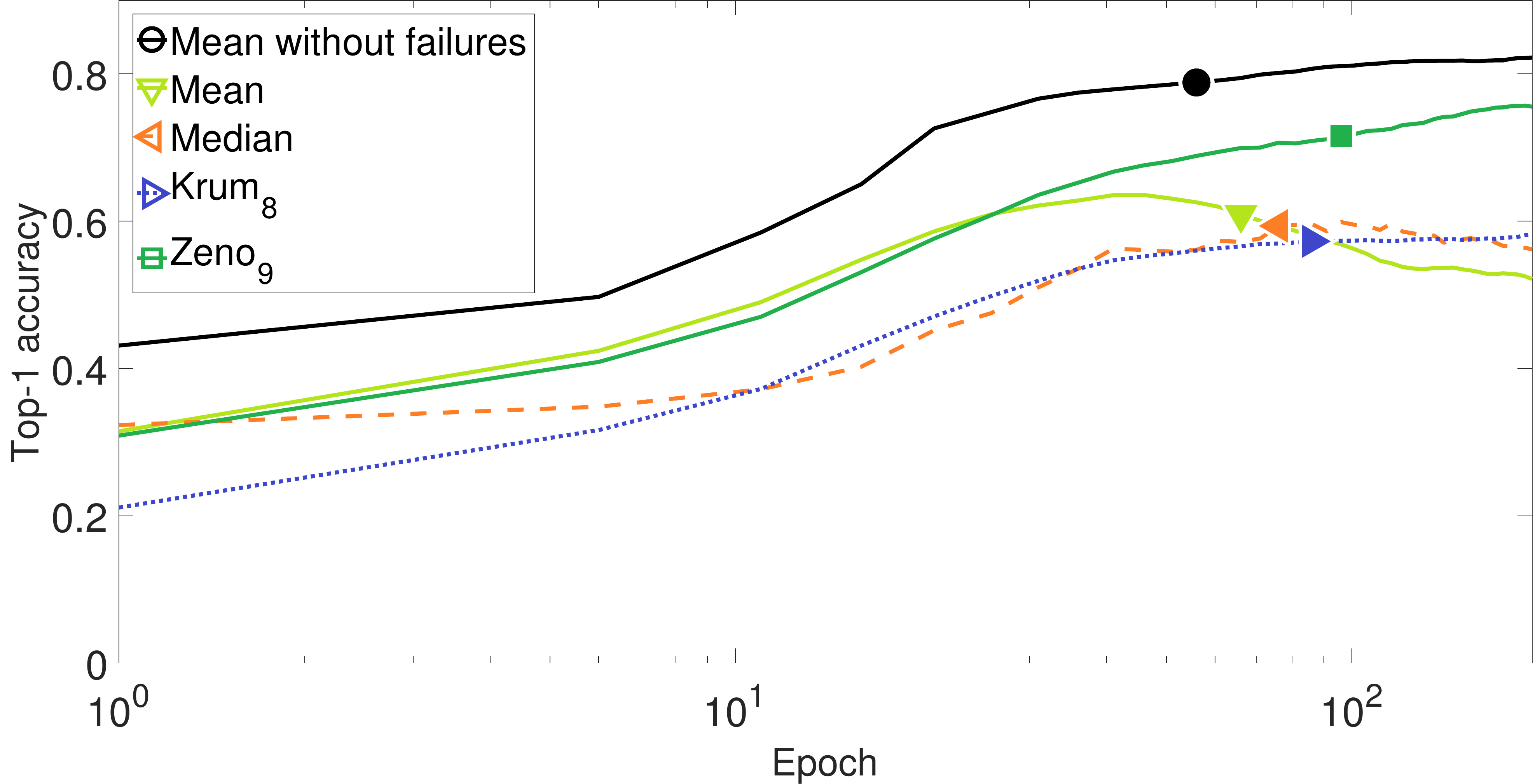}}
\subfigure[Cross entropy on training set, with $q=8$]{\includegraphics[width=0.49\textwidth,height=3.9cm]{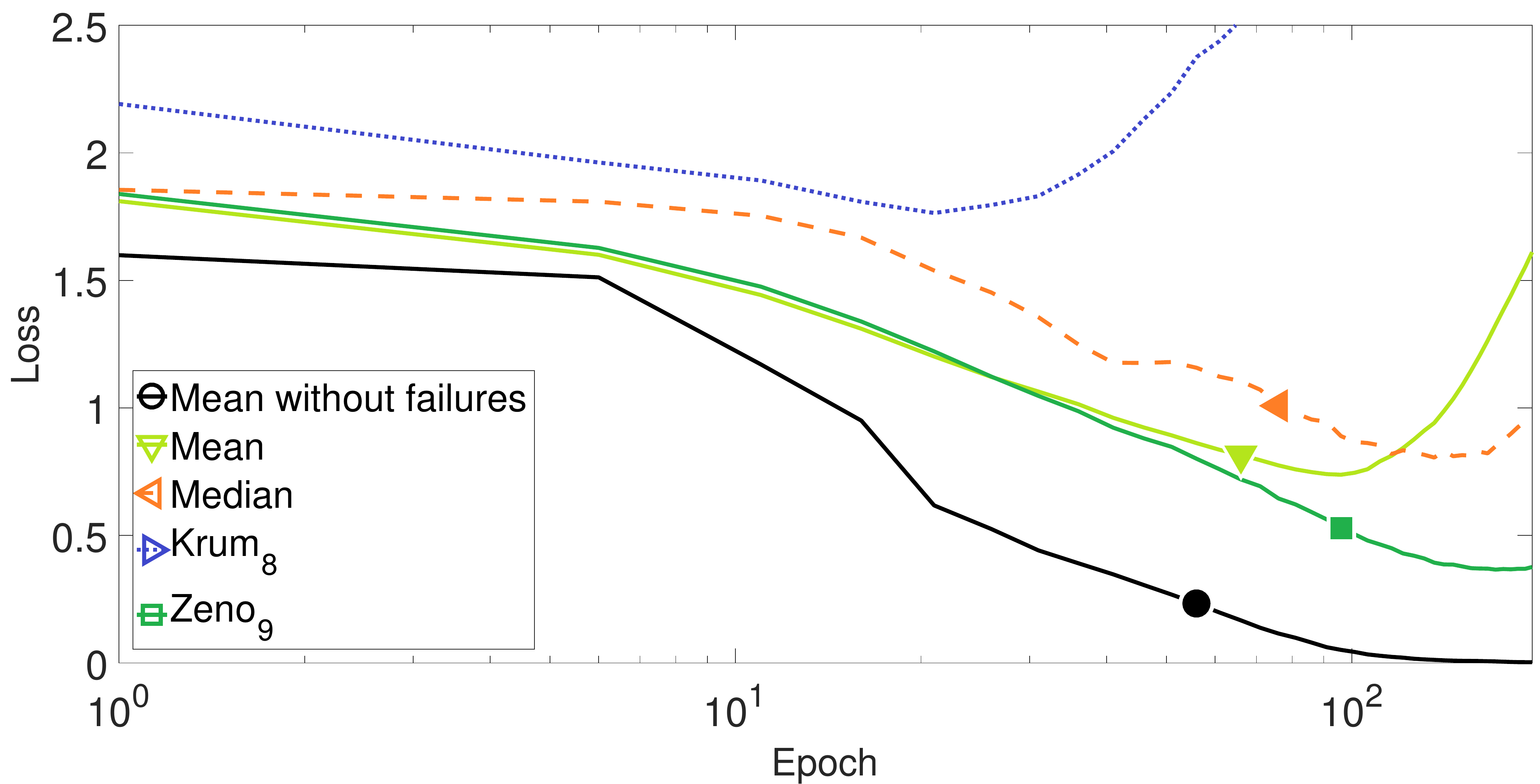}}
\subfigure[Top-1 accuracy on testing set, with $q=12$]{\includegraphics[width=0.49\textwidth,height=3.9cm]{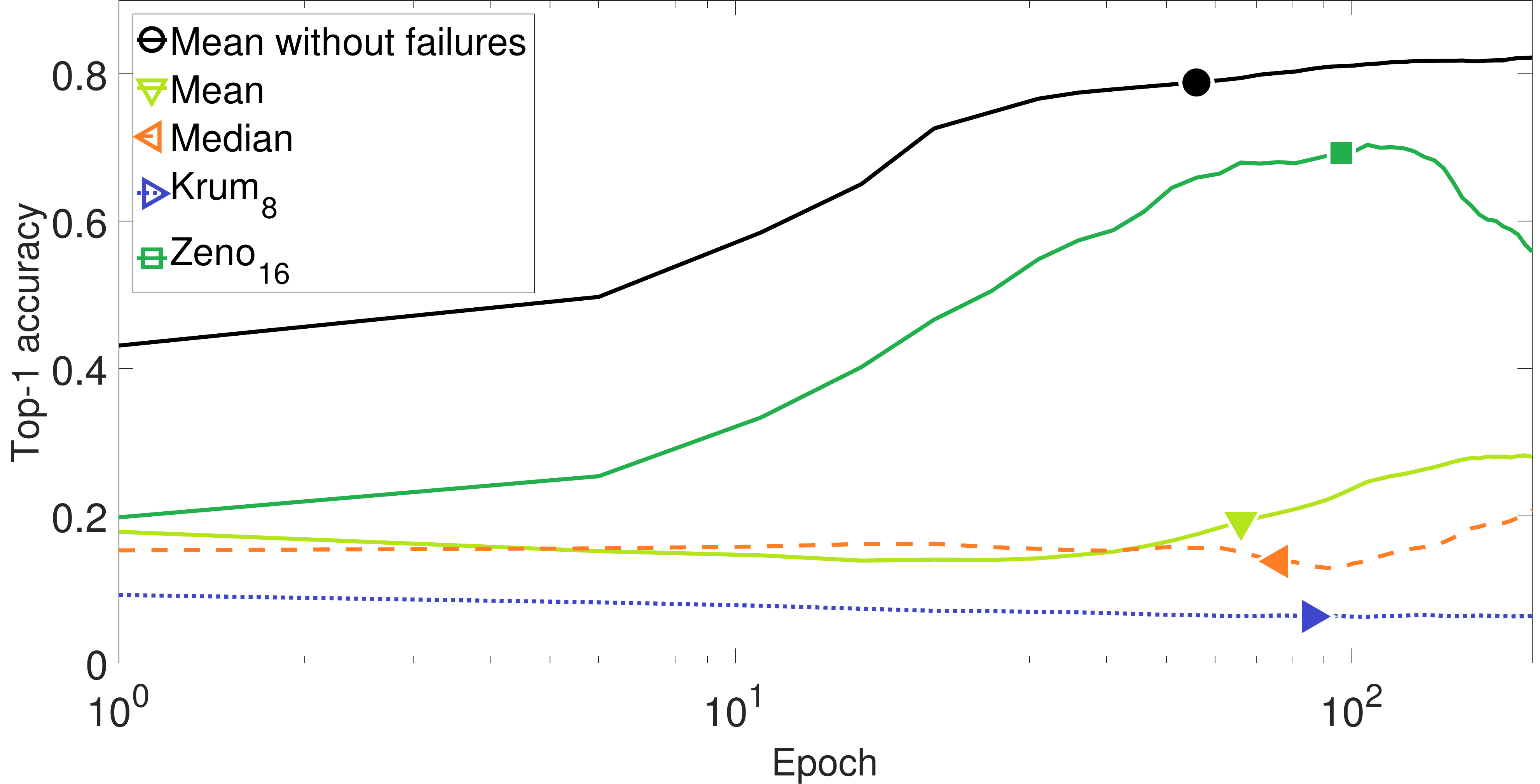}}
\subfigure[Cross entropy on training set, with $q=12$]{\includegraphics[width=0.49\textwidth,height=3.9cm]{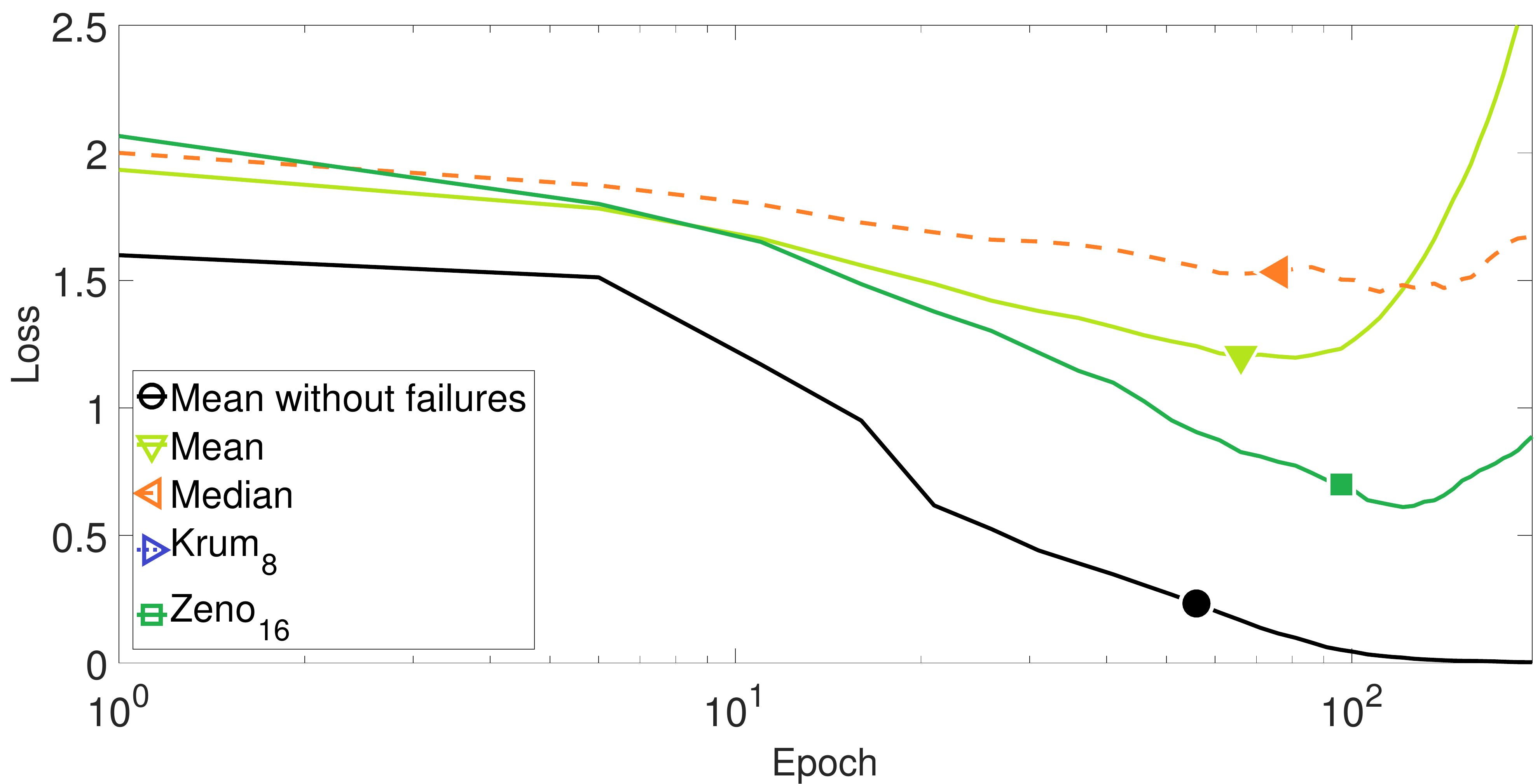}}
\caption{Convergence on disjoint~(non-i.i.d.) training data, with label-flipping failures. Batch size on the workers is $100$. Batch size of \texttt{Zeno} is $n_r=4$. $\rho=0.0005$. $\gamma = 0.05$.  Each epoch has 25 iterations. \texttt{Zeno} outperforms all the baselines, especially when $q=12$.}
\label{fig:noniid_labelflip}
\end{figure*}
\begin{figure*}[htb!]
\centering
\subfigure[Top-1 accuracy on testing set, with $q=8$]{\includegraphics[width=0.49\textwidth,height=3.6cm]{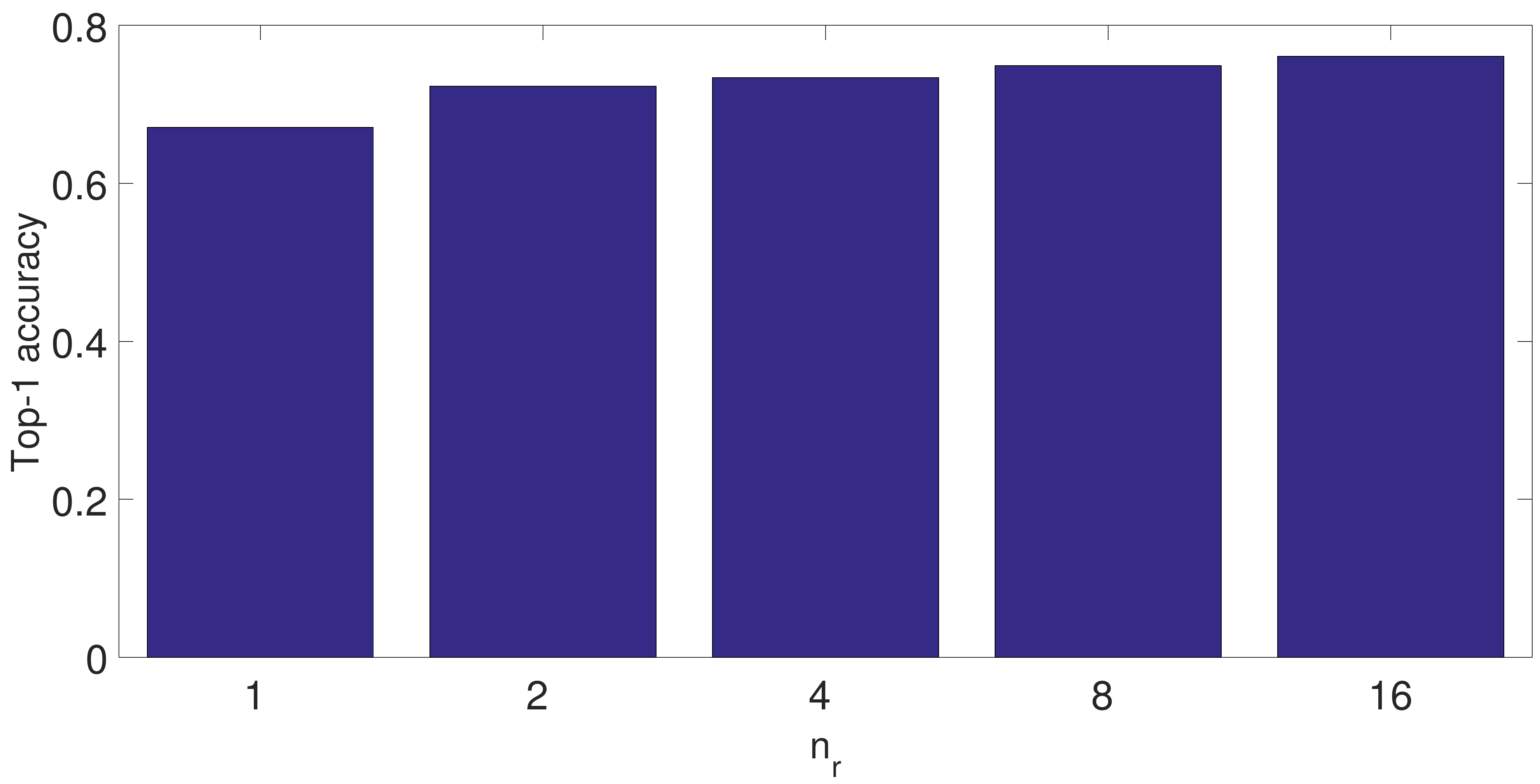}}
\subfigure[Cross entropy on training set, with $q=8$]{\includegraphics[width=0.49\textwidth,height=3.6cm]{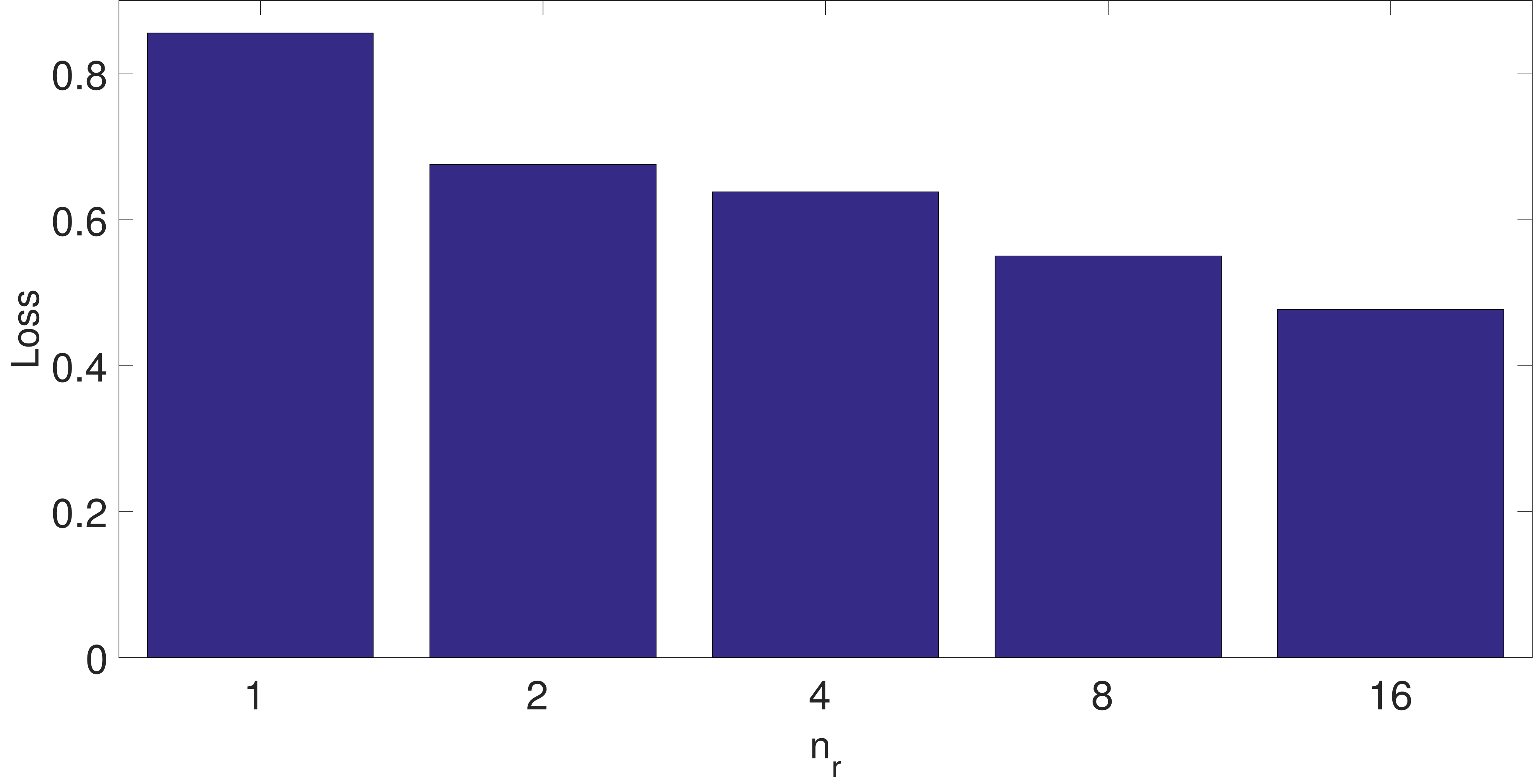}}
\caption{Convergence on i.i.d. training data, with label-flipping failures, $q=8$. Batch size on the workers is $100$. $\gamma = 0.1$.  Each epoch has 25 iterations. $n_r$ is tuned.}
\label{fig:bs}
\end{figure*}

\section{Experiments}

In this section, we evaluate the fault tolerance of the proposed algorithm. We summarize our results here:
\setitemize[0]{leftmargin=*}
\begin{itemize} 
\item Compared to the baselines, \texttt{Zeno} shows better convergence with more faulty workers than non-faulty ones.
\item \texttt{Zeno} is robust to the choices of the hyperparameters, including the Zeno batch size $n_r$, the weight $\rho$, and the number of trimmed elements $b$.
\item \texttt{Zeno} also works when training with disjoint local data.
\end{itemize}
%We first show the resilience to two different attacks, and then show how the sensitivity to the hyperparameters of \textit{Zeno}. The corresponding hyperparameters are the batch size $n_r$ for evaluating the stochastic descendant score in Definition~\ref{def:score}, the weight $\rho$, and the number of trimmed workers $b$.

\subsection{Datasets and Evaluation Metrics}
We conduct experiments on benchmark CIFAR-10 image classification dataset~\citep{krizhevsky2009learning}, which is composed of 50k images for training and 10k images for testing. We use convolutional neural network~(CNN) with 4 convolutional layers followed by 1 fully connected layer. The detailed network architecture can be found in \url{https://github.com/xcgoner/icml2019_zeno}.
In each experiment, we launch 20 worker processes. We repeat each experiment 10 times and take the average. We use top-1 accuracy on the testing set and the cross-entropy loss function on the training set as the evaluation metrics. 

\subsubsection{Baselines}
We use the averaging without failures/attacks as the gold standard, which is referred to as \texttt{Mean without failures}. Note that this method is not affected by $b$ or $q$. The baseline aggregation rules are \texttt{Mean}, \texttt{Median}, and \texttt{Krum} as defined below.

\begin{definition}(Median~\cite{yin2018byzantine})
\label{def:marmed}
We define the marginal median aggregation rule $\median(\cdot)$ as
$
med = \median(\{\tilde{v}_i: i \in [m]\}),
$
where for any $j\in[d]$, the $j$th dimension of $med$ is $med_j = median\left(\{(\tilde{v}_1)_j, \ldots, (\tilde{v}_m)_j\}\right)$, $(\tilde{v}_i)_j$ is the $j$th dimension of the vector $\tilde{v}_i$, $median(\cdot)$ is the one-dimensional median.  
\end{definition}

\begin{definition}(Krum~\cite{blanchard2017machine})
\label{def:krum}
\begin{align*}
&\krum_b(\{\tilde{v}_i: i \in [m]\}) = \tilde{v}_k, \quad
k = \argmin_{i \in [m]} \sum_{i \rightarrow j} \| \tilde{v}_i - \tilde{v}_j \|^2,
\end{align*}
where $i \rightarrow j$ is the indices of the $m-b-2$ nearest neighbours of $\tilde{v}_i$ in $\{\tilde{v}_i: i \in [m]\}$ measured by Euclidean distances.
\end{definition}

Note that \texttt{Krum} requires $2b+2 < m$. Thus, $b=8$ is the best we can take.

\subsection{No Failure}
We first test the convergence when there are no failures. In all the experiments, we take the learning rate $\gamma=0.1$, worker batch size $100$, \texttt{Zeno} batch size $n_r=4$, and $\rho = 0.0005$. Each worker computes the gradients on i.i.d. samples. For both \texttt{Krum} and \texttt{Zeno}, we take $b=4$. The result is shown in Figure~\ref{fig:iid_nobyz}. We can see that \texttt{Zeno} converges as fast as \texttt{Mean}. \texttt{Krum} converges slightly slower, but the convergence rate is acceptable.

\subsection{Label-flipping Failure}
In this section, we test the fault tolerance to label-flipping failures. When such  failures happen, the workers compute the gradients based on the training data with ``flipped" labels, i.e., any $label \in \{0, \ldots, 9\}$, is replaced by $9 - label$. Such failures/attacks can be caused by data poisoning or software failures. 

In all the experiments, we take the learning rate $\gamma=0.1$, worker batch size $100$, \texttt{Zeno} batch size $n_r=4$, and $\rho = 0.0005$. Each non-faulty worker computes the gradients on i.i.d. samples.

The result is shown in Figure~\ref{fig:iid_labelflip}. As expected, \texttt{Zeno} can tolerate more than half faulty gradients. When $q=8$, \texttt{Zeno} preforms similar to \texttt{Krum}. When $q=12$, \texttt{Zeno} preforms much better than the baselines. When there are faulty gradients, \texttt{Zeno} converges slower, but still has better convergence rates than the baselines. 

\subsection{Bit-flipping Failure}
In this section, we test the fault tolerance to a more severe kind of failure. Here, the bits that control the sign of the floating numbers are flipped, e.g.,  due to some hardware failure. A faulty worker pushes the negative gradient instead of the true gradient to the servers. To make the failure even worse, one of the faulty gradients is copied to and overwrites the other faulty gradients, which means that all the faulty gradients have the same value.

In all the experiments, we take the learning rate $\gamma=0.1$, worker batch size $100$, \texttt{Zeno} batch size $n_r=4$, and $\rho = 0.0005$. Each non-faulty worker computes the gradients on i.i.d. samples.

The result is shown in Figure~\ref{fig:iid_signflip}. As expected, \texttt{Zeno} can tolerate more than half faulty gradients. Surprisingly, \texttt{Mean} performs well when $q=8$. We will discuss this phenomenon in Section~\ref{subsec:discussion}. \texttt{Zeno} outperforms all the baselines. When $q=12$, \texttt{Zeno} is the only strategy which avoids catastrophic divergence. \texttt{Zeno} converges slower, but still has better convergence than the baselines. 

\subsection{Disjoint Local Training Data}
In volunteer computing~\cite{Meeds2015MLitBML,Miura2015ImplementationOA}, it is reasonable for the coordinator to assign disjoint tasks/datasets to different workers. As a result, each worker draws training samples from different datasets. The server is still aware of the entire dataset. We conduct experiments in such scenario, as discussed in Corollary~\ref{cor:disjoint}. We test \texttt{Zeno} under label-flipping failures. The results are shown in Figure~\ref{fig:noniid_labelflip}. Due to the non-i.i.d. setting, it is more difficult to distinguish faulty gradients from non-faulty ones. In such bad cases, \texttt{Zeno} can still make reasonable progress, while the baselines, especially \texttt{Krum}, performs much worse.

\subsection{Hyperparameter Sensitivity}
In Figure~\ref{fig:bs}, we show the performance of \texttt{Zeno} with different batch size $n_r$. Larger $n_r$ improves the convergence, but the gap is not significant. $n_r = 1$ still works. \texttt{Zeno} is also robust to different choices of the other hyperparameters $\rho$ and $b$. The experiments can be found in the appendix.

\subsection{Discussion}
\label{subsec:discussion}

An interesting observation is that, when $q=8$, \texttt{Mean} seems to have good performance, while it is not supposed to be fault-tolerant. The reason is that both label-flipping and bit-flipping failures do not change the magnitude of the gradients. When the number of faulty gradients $q$ is less than half, it is possible that the faulty gradients are cancelled out by the non-faulty ones. However, when the magnitude is enlarged, \texttt{Mean} will fail, as pointed out in \citet{xie2018phocas}.

In general, we find that \texttt{Zeno} is more robust than the current state of the art. When the faulty workers dominate, \texttt{Zeno} is the only aggregator that converges in all experiments. When the correct workers dominate, \texttt{Median} can be an alternative with cheap computation. 

The computational complexity of \texttt{Zeno} depends on the complexity of inference and the Zeno batch size $n_r$. These additional hyperparameters make direct comparison to standard methods more challenging. If we take the approximation that the computational complexity of inference is linear to the number of parameters, then we can roughly compare the time complexity to the baselines. Compared to \texttt{Median}, \texttt{Zeno} is computationally more expensive by the factor of $n_r = 4$. However, compared to \texttt{Krum}, which requires $20 \times 19 / 2 = 190$ times of $\mathcal{O}(d)$ operators, \texttt{Zeno} only needs $21 \times 4 = 84$ times of $\mathcal{O}(d)$ operators. Furthermore, since the batch size on the workers is $100$, the computation required on the server is less than that of one worker, which does not cancel out the computational improvements due to data parallelism. The additional computation is the cost that we have to pay for better robustness.

Another interesting observation is that, although \texttt{Krum} is the state-of-the-art algorithm, it does not perform as well as expected under our designed failures. The reason is that \texttt{Krum} requires the assumption that $c \sigma < \|g\|$ for convergence, where $c$ is a general constant, $\sigma$ is the maximal variance of the gradients, and $g$ is the gradient. Note that $\|g\| \rightarrow 0$ when SGD converges to a critical point. Thus, such assumption is never guaranteed to be satisfied, if the variance is large. Furthermore, the better SGD converges, the less likely such assumption can be satisfied~(more details of this issue can be found in \citet{xie2019fall}). 
% We can provide an 1-dimensional toy example to show what will happen when the variance is large enough. Suppose there are $4$ non-faulty gradients $\{0.2, 0.4, 1.6, 1.8\}$ with the mean $1$, and $2$ faulty gradient $\{-1, -1\}$ with the opposite mean $-1$. According to Definition~\ref{def:krum}, taking $b=2$, $\krum_2(\{-1, -1, 0.2, 0.4, 1.6, 1.8\})  = -1$, which means that \texttt{Krum} chooses the faulty gradient when variance is large enough. Note that it is easier to have larger variances when the dimension $d$ gets higher.

\section{Conclusion}

We propose a novel aggregation rule for synchronous SGD, which requires a weak assumption that there is at least one honest worker. The algorithm has provable convergence. Our empirical results show good performance in practice. We will apply the proposed method to asynchronous SGD in future work.

\section*{Acknowledgements}
This work was funded in part by NSF CNS 1409416, by a gift from Microsoft, and by computational resources donated by Intel, AWS, and Microsoft Azure.

\bibliography{zeno}
\bibliographystyle{icml2019}

\appendix

\newpage
\onecolumn

\vspace*{0.1cm}
\begin{center}
	\Large\textbf{Appendix}
\end{center}
\vspace*{0.1cm}

\section{Proofs}

\subsection{Preliminaries}
We use the following lemma to bound the aggregated vectors.
\begin{lemma} (Bounded Score)
\label{lem:score}
Without loss of generality, we denote the $m-q$ correct elements in $\{\tilde{v}_i: i \in [m]\}$ as $\{v_i: i \in [m-q]\}$. Sorting the correct vectors by the stochastic descendant score, we obtain $\{v_{(i)}: i \in [m-q]\}$. Then, we have the following inequality:
\begin{align*}
Score_{\gamma, \rho}(\tilde{v}_{(i)}, x) \geq Score_{\gamma, \rho}(v_{(i)}, x), \forall i \in [m-q],
\end{align*}
or, by flipping the signs on both sides, it is equivalent to  
\begin{align*}
f_r(x - \gamma \tilde{v}_{(i)}) - f_r(x) + \rho \| \tilde{v}_{(i)} \|^2 \leq 
f_r(x - \gamma v_{(i)}) - f_r(x) + \rho \| v_{(i)} \|^2,
\forall i \in [m-q],
\end{align*}
\end{lemma}
\begin{proof}
We prove the lemma by contradiction.

Assume that $Score_{\gamma, \rho}(\tilde{v}_{(i)}, x) < Score_{\gamma, \rho}(v_{(i)}, x)$. Thus, there are $i$ correct vectors having greater scores than $\tilde{v}_{(i)}$. However, because $\tilde{v}_{(i)}$ is the $i$th element in $\{\tilde{v}_{(i)}: i \in [m]\}$, there should be at most $i-1$ vectors having greater scores than it, which yields a contradiction.
\end{proof}

\subsection{Convergence guarantees}
For general non-strongly convex functions and non-convex functions, we provide the following convergence guarantees.

\setcounter{theorem}{0}
\begin{theorem}
For $\forall x \in \R^d$, denote
\begin{align*}
\tilde{v}_i = 
\begin{cases}
* & \mbox{$i$th worker is Byzantine}, \\
\nabla F_i(x) & \mbox{otherwise,}
\end{cases}
\end{align*}
where $i \in [m]$, and $\bar{\tilde{v}} = \zeno_b(\{\tilde{v}_i: i \in [m]\})$.
Taking $\gamma \leq \frac{1}{L}$, and $\rho = \frac{\beta \gamma^2}{2}$, where 
\begin{align*}
\begin{cases}
\beta = 0, &\mbox{\quad if $\mu \geq 0$;}\\
\beta \geq |\mu|, &\mbox{\quad otherwise.}
\end{cases}
\end{align*} we have 
\begin{align*}
\E \left[ F(x - \gamma \bar{\tilde{v}}) \right] - F(x) 
\leq -\frac{\gamma}{2} \| \nabla F(x) \|^2 + \frac{\gamma(b-q + 1)(m-q)V}{(m-b)^2} + \frac{(L+\beta)\gamma^2 G}{2}.
\end{align*}
\end{theorem}
\begin{proof}
Without loss of generality, we denote the $m-q$ correct elements in $\{\tilde{v}_i: i \in [m]\}$ as $\{v_i: i \in [m-q]\}$, where $\E[v_i] = \nabla F(x)$. Sorting the correct vectors by the online descendant score, we obtain $\{v_{(i)}: i \in [m-q]\}$. We also sort $\tilde{v}_i$ by the online descendant score and obtain $\{\tilde{v}_{(i)}: i \in [m]\}$.

According to the definition, $\bar{\tilde{v}} = \zeno_b(\{\tilde{v}_i: i \in [m]\}) =  \frac{1}{m-b} \sum_{i=1}^{m-b} \tilde{v}_{(i)}$. Furthermore, we denote $\bar{v} = \frac{1}{m-b} \sum_{i=1}^{m-b} v_{(i)}$.

Using Assumption~\ref{asm:smooth_cvx}, we have
\begin{align*}
& f_r(x - \gamma \tilde{v}_{(i)}) \geq f_r(x - \gamma \bar{\tilde{v}}) + \ip{\nabla f_r(x - \gamma \bar{\tilde{v}})}{ \gamma (\bar{\tilde{v}} - \tilde{v}_{(i)})} + \frac{\mu \gamma^2}{2} \| \bar{\tilde{v}} - \tilde{v}_{(i)} \|^2,
\end{align*}
for $\forall i \in [m-b]$.

By summing up, we have 
\begin{align}
& \frac{1}{m-b} \sum_{i=1}^{m-b} f_r(x - \gamma \tilde{v}_{(i)}) \geq f_r(x - \gamma \bar{\tilde{v}}) + \frac{\mu \gamma^2}{2(m-b)}  \sum_{i=1}^{m-b} \| \bar{\tilde{v}} - \tilde{v}_{(i)} \|^2. \label{equ:str_cvx}
\end{align}

Using Lemma~\ref{lem:score}, we have 
\begin{align*}
f_r(x - \gamma \tilde{v}_{(i)}) + \rho \|\tilde{v}_{(i)}\|^2 \leq f_r(x - \gamma v_{(i)}) + \rho \|v_{(i)}\|^2,
\end{align*}
for $\forall i \in [m-b]$.

Combined with Equation~\ref{equ:str_cvx}, we have 
\begin{align*}
&f_r(x - \gamma \bar{\tilde{v}}) \\
&\leq \frac{1}{m-b} \sum_{i=1}^{m-b} f_r(x - \gamma \tilde{v}_{(i)}) - \frac{\mu \gamma^2}{2(m-b)}  \sum_{i=1}^{m-b} \| \bar{\tilde{v}} - \tilde{v}_{(i)} \|^2 \\
&\leq \frac{1}{m-b} \sum_{i=1}^{m-b} f_r(x - \gamma v_{(i)}) + \frac{\rho}{m-b} \sum_{i=1}^{m-b} \left[ \|v_{(i)}\|^2 - \|\tilde{v}_{(i)}\|^2 \right] - \frac{\mu \gamma^2}{2(m-b)}  \sum_{i=1}^{m-b} \| \bar{\tilde{v}} - \tilde{v}_{(i)} \|^2.
\end{align*}

We take $\rho = \frac{\beta \gamma^2}{2}$, where 
\begin{align*}
\begin{cases}
\beta = 0, &\mbox{\quad if $\mu \geq 0$;}\\
\beta \geq |\mu|, &\mbox{\quad otherwise.}
\end{cases}
\end{align*}
Thus, if $\mu \geq 0$, we have $\rho = 0$, which implies that
\begin{align*}
\frac{\rho}{m-b} \sum_{i=1}^{m-b} \left[ \|v_{(i)}\|^2 - \|\tilde{v}_{(i)}\|^2 \right] - \frac{\mu \gamma^2}{2(m-b)}  \sum_{i=1}^{m-b} \| \bar{\tilde{v}} - \tilde{v}_{(i)} \|^2 \leq \frac{\beta \gamma^2}{2(m-b)} \sum_{i=1}^{m-b} \|v_{(i)}\|^2.
\end{align*}
Also, if $\mu < 0$, since $\beta \geq -\mu$, we have 
\begin{align*}
&\frac{\rho}{m-b} \sum_{i=1}^{m-b} \left[ \|v_{(i)}\|^2 - \|\tilde{v}_{(i)}\|^2 \right] - \frac{\mu \gamma^2}{2(m-b)}  \sum_{i=1}^{m-b} \| \bar{\tilde{v}} - \tilde{v}_{(i)} \|^2 \\
&= \frac{\beta \gamma^2}{2(m-b)} \sum_{i=1}^{m-b} \left[ \|v_{(i)}\|^2 - \|\tilde{v}_{(i)}\|^2 \right] - \frac{\mu \gamma^2}{2(m-b)}  \sum_{i=1}^{m-b} \left[ \|\tilde{v}_{(i)}\|^2 - \|\bar{\tilde{v}}\|^2 \right] \\
&= \frac{\beta \gamma^2}{2(m-b)} \sum_{i=1}^{m-b} \|v_{(i)}\|^2 + \frac{(-\beta-\mu) \gamma^2}{2(m-b)} \sum_{i=1}^{m-b} \|\tilde{v}_{(i)}\|^2  + \frac{\mu \gamma^2}{2(m-b)}  \sum_{i=1}^{m-b} \|\bar{\tilde{v}}\|^2 \\
&\leq \frac{\beta \gamma^2}{2(m-b)} \sum_{i=1}^{m-b} \|v_{(i)}\|^2.
\end{align*}

Thus, we have
\begin{align*}
&f_r(x - \gamma \bar{\tilde{v}}) \\
&\leq \frac{1}{m-b} \sum_{i=1}^{m-b} f_r(x - \gamma v_{(i)}) + \frac{\rho}{m-b} \sum_{i=1}^{m-b} \left[ \|v_{(i)}\|^2 - \|\tilde{v}_{(i)}\|^2 \right] - \frac{\mu \gamma^2}{2(m-b)}  \sum_{i=1}^{m-b} \| \bar{\tilde{v}} - \tilde{v}_{(i)} \|^2 \\
&\leq \frac{1}{m-b} \sum_{i=1}^{m-b} f_r(x - \gamma v_{(i)}) + \frac{\beta \gamma^2}{2(m-b)} \sum_{i=1}^{m-b} \|v_{(i)}\|^2.
\end{align*}

Using the $L$-smoothness, we have 
\begin{align*}
f_r(x - \gamma v_{(i)}) \leq f_r(x - \gamma \bar{v}) + \ip{\nabla f_r(x - \gamma \bar{v})}{\gamma ( \bar{v} - v_{(i)} )} + \frac{L\gamma^2}{2} \| \bar{v} - v_{(i)} \|^2,
\end{align*}
for $\forall i \in [m-b]$.
By summing up, we have 
\begin{align*}
&\frac{1}{m-b} \sum_{i=1}^{m-b} f_r(x - \gamma v_{(i)}) \\
&\leq f_r(x - \gamma \bar{v}) + \frac{L\gamma^2}{2(m-b)} \sum_{i=1}^{m-b} \| \bar{v} - v_{(i)} \|^2 \\
&\leq f_r(x - \gamma \bar{v}) + \frac{L\gamma^2}{2(m-b)} \sum_{i=1}^{m-b} \| v_{(i)} \|^2.
\end{align*}

Thus, we have 
\begin{align*}
&f_r(x - \gamma \bar{\tilde{v}}) \\
&\leq \frac{1}{m-b} \sum_{i=1}^{m-b} f_r(x - \gamma v_{(i)}) + \frac{\beta \gamma^2}{2(m-b)} \sum_{i=1}^{m-b} \|v_{(i)}\|^2 \\
&\leq f_r(x - \gamma \bar{v}) + \frac{(L+\beta)\gamma^2}{2(m-b)} \sum_{i=1}^{m-b} \| v_{(i)} \|^2.
\end{align*}

Again, using the $L$-smoothness and taking $\gamma \leq \frac{1}{L}$, we have 
\begin{align*}
&f_r(x - \gamma \bar{v}) \\
&\leq f_r(x) + \ip{\nabla f_r(x)}{ -\gamma \bar{v}} + \frac{L\gamma^2}{2} \|\bar{v}\|^2 \\
&\leq f_r(x) + \ip{\nabla f_r(x)}{ -\gamma \bar{v}} + \frac{\gamma}{2} \|\bar{v}\|^2
\end{align*}

Thus, we have 
\begin{align*}
&f_r(x - \gamma \bar{\tilde{v}}) - f_r(x) \\
&\leq f_r(x - \gamma \bar{v}) - f_r(x) + \frac{(L+\beta)\gamma^2}{2(m-b)} \sum_{i=1}^{m-b} \| v_{(i)} \|^2 \\
&\leq \ip{\nabla f_r(x)}{ -\gamma \bar{v}} + \frac{\gamma}{2} \|\bar{v}\|^2 + \frac{(L+\beta)\gamma^2}{2(m-b)} \sum_{i=1}^{m-b} \| v_{(i)} \|^2.
\end{align*}

Conditional on $\tilde{v}_{(i)}$'s, taking expectation w.r.t. $f_r$ on both sides, we have 
\begin{align*}
&F(x - \gamma \bar{\tilde{v}}) - F(x) \\
&\leq \ip{\nabla F(x)}{ -\gamma \bar{v}} + \frac{\gamma}{2} \|\bar{v}\|^2 + \frac{(L+\beta)\gamma^2}{2(m-b)} \sum_{i=1}^{m-b} \| v_{(i)} \|^2 \\
&= -\frac{\gamma}{2} \| \nabla F(x) \|^2 + \frac{\gamma}{2} \| \nabla F(x) - \bar{v} \|^2 + \frac{(L+\beta)\gamma^2}{2(m-b)} \sum_{i=1}^{m-b} \| v_{(i)} \|^2.
\end{align*}

Now, taking the expectation w.r.t. $\tilde{v}_{(i)}$'s on both sides and using $\E\|v_{(i)}\|^2 \leq G$, we have 
\begin{align*}
&\E \left[ F(x - \gamma \bar{\tilde{v}}) \right] - F(x) \\
&\leq -\frac{\gamma}{2} \| \nabla F(x) \|^2 + \frac{\gamma}{2} \E \| \nabla F(x) - \bar{v} \|^2 + \frac{(L+\beta)\gamma^2}{2(m-b)} \sum_{i=1}^{m-b} \E\| v_{(i)} \|^2 \\
&\leq -\frac{\gamma}{2} \| \nabla F(x) \|^2 + \frac{\gamma}{2} \E \| \nabla F(x) - \bar{v} \|^2 + \frac{(L+\beta)\gamma^2 G}{2}.
\end{align*}

Now we just need to bound $\E \| \nabla F(x) - \bar{v} \|^2$. For convenience, we denote $g = \nabla F(x)$.
Note that for arbitrary subset $\mathcal{S} \subseteq [m-q]$, $|\mathcal{S}| = m-b$, we have the following bound:
\begin{align*}
&\E\left\| \frac{\sum_{i \in \mathcal{S}} (v_i-g)}{m-b } \right\|^2 \\
&= \E\left\| \frac{\sum_{i \in [m-q]} (v_i-g) - \sum_{i \notin \mathcal{S}} (v_i-g)}{m-b} \right\|^2 \\
&\leq 2\E\left\| \frac{\sum_{i \in [m-q]} (v_i-g)}{m-b} \right\|^2 + 2\E\left\| \frac{\sum_{i \notin \mathcal{S}} (v_i-g)}{m-b} \right\|^2 \\
&= \frac{2(m-q)^2}{(m-b)^2} \E\left\| \frac{\sum_{i \in [m-q]} (v_i-g)}{m-q} \right\|^2 + \frac{2(b-q)^2}{(m-b)^2} \E\left\| \frac{\sum_{i \notin \mathcal{S}} (v_i-g)}{b-q} \right\|^2 \\
&\leq \frac{2(m-q)^2}{(m-b)^2} \frac{V}{m-q} + \frac{2(b-q)^2}{(m-b)^2} \frac{\sum_{i \in [m-q]} \|v_i-g\|^2}{b-q} \\
&\leq \frac{2(m-q)^2}{(m-b)^2} \frac{V}{m-q} + \frac{2(b-q)^2}{(m-b)^2} \frac{(m-q) V}{b-q} \\
&= \frac{2(b-q + 1)(m-q)V}{(m-b)^2}.
\end{align*}

Putting all the ingredients together, we obtain the desired result
\begin{align*}
&\E \left[ F(x - \gamma \bar{\tilde{v}}) \right] - F(x) \\
&\leq -\frac{\gamma}{2} \| \nabla F(x) \|^2 + \frac{\gamma(b-q + 1)(m-q)V}{(m-b)^2} + \frac{(L+\beta)\gamma^2 G}{2}.
\end{align*}

\end{proof}

\setcounter{corollary}{0}
\begin{corollary}
Take $\gamma = \frac{1}{L \sqrt{T}}$, and $\rho = \frac{\beta\gamma^2}{2}$, where $\beta$ is the same as in Theorem~\ref{thm:byz_step}. Using $\zeno$, after $T$ iterations, we have
\begin{align*}
&\frac{\sum_{t=0}^{T-1} \E \| \nabla F(x^t) \|^2}{T} \\
&\leq \left[ 2L \left( F(x^0) - F(x^*) \right) + \frac{(L + \beta)G}{L} \right] \frac{1}{\sqrt{T}} + \frac{2(b-q + 1)(m-q)V}{(m-b)^2} \\
&= \mathcal{O}\left( \frac{1}{\sqrt{T}} \right) + \mathcal{O} \left( \frac{(b-q + 1)(m-q)}{(m-b)^2} \right).
\end{align*}
\end{corollary}
\begin{proof}
Taking $x = x^t$, $x - \gamma \zeno_b(\{\tilde{v}_i: i \in [m]\}) = x^{t+1}$, 
using Theorem~\ref{thm:byz_step}, we have 
\begin{align*}
&\E \left[ F(x^{t+1}) \right] - F(x^t) \\
&\leq -\frac{\gamma}{2} \| \nabla F(x^t) \|^2 + \frac{\gamma(b-q + 1)(m-q)V}{(m-b)^2} + \frac{(L+\beta)\gamma^2 G}{2}.
\end{align*}
By telescoping and taking total expectation, we have 
\begin{align*}
&\E \left[ F(x^T) \right] - F(x^0) \\
&\leq -\frac{\gamma}{2} \sum_{t=0}^{T-1} \E \| \nabla F(x^t) \|^2 + \frac{\gamma(b-q + 1)(m-q)V T}{(m-b)^2} + \frac{(L+\beta)\gamma^2 G T}{2}.
\end{align*}
Taking $\gamma = \frac{1}{L \sqrt{T}}$, we have
\begin{align*}
&\frac{\sum_{t=0}^{T-1} \E \| \nabla F(x^t) \|^2}{T} \\
&\leq \frac{2L \left[ F(x^0) - F(x^T) \right]}{\sqrt{T}} + \frac{2(b-q + 1)(m-q)V}{(m-b)^2} + \frac{(L + \beta)G}{L \sqrt{T}} \\
&\leq \frac{2L \left[ F(x^0) - F(x^*) \right]}{\sqrt{T}} + \frac{2(b-q + 1)(m-q)V}{(m-b)^2} + \frac{(L + \beta)G}{L \sqrt{T}} \\
&= \mathcal{O}\left( \frac{1}{\sqrt{T}} \right) + \mathcal{O} \left( \frac{(b-q + 1)(m-q)}{(m-b)^2} \right).
\end{align*}

\end{proof}

\begin{corollary}
Assume that 
\begin{align*}
F(x) = \frac{1}{m} \sum_{i \in [m]} \E \left[ F_i(x) \right],
\end{align*}
and 
\begin{align*}
\E \left[ F_i(x) \right] \neq \E \left[ F_j(x) \right],
\end{align*}
for $\forall i, j \in [m]$, $i \neq j$. For the stochastic descendant score, we still have $\E \left[ f_r(x) \right] = F(x)$. Assumption~\ref{asm:score}, \ref{asm:smooth_cvx}, and \ref{asm:variance} still hold.
Take $\gamma = \frac{1}{L \sqrt{T}}$, and $\rho = \frac{\beta\gamma^2}{2}$, where $\beta$ is the same as in Theorem~\ref{thm:byz_step}. Using $\zeno$, after $T$ iterations, we have
\begin{align*}
&\frac{\sum_{t=0}^{T-1} \E \| \nabla F(x^t) \|^2}{T} \\
&\leq \frac{2L \left[ F(x^0) - F(x^*) \right]}{\sqrt{T}} +  \frac{4V}{m} + \frac{4bG}{m} + \frac{2b^2(m-q)G}{m^2(m-b)} + \frac{(L + \beta)G}{L \sqrt{T}} \\
&= \mathcal{O}\left( \frac{1}{\sqrt{T}} \right) + \mathcal{O}\left(\frac{b}{m}\right) + \mathcal{O}\left(\frac{b^2(m-q)}{m^2(m-b)}\right).
\end{align*}
\end{corollary}

\begin{proof}
Similar to the proof of Theorem~\ref{thm:byz_step}, we define $\bar{\tilde{v}} =  \zeno_b(\{\tilde{v}_i: i \in [m]\})$. Thus, reusing the proof in Theorem~\ref{thm:byz_step}, we have 
\begin{align*}
F(x - \gamma \bar{\tilde{v}}) - F(x) \leq -\frac{\gamma}{2} \| \nabla F(x) \|^2 + \frac{\gamma}{2} \| \nabla F(x) - \bar{v} \|^2 + \frac{(L+\beta)\gamma^2}{2(m-b)} \sum_{i=1}^{m-b} \| v_{(i)} \|^2.
\end{align*}

Now, taking the expectation w.r.t. $\tilde{v}_{(i)}$'s on both sides and using $\E\|v_{(i)}\|^2 \leq G$, we have 
\begin{align*}
&\E \left[ F(x - \gamma \bar{\tilde{v}}) \right] - F(x) \\
&\leq -\frac{\gamma}{2} \| \nabla F(x) \|^2 + \frac{\gamma}{2} \E \| \nabla F(x) - \bar{v} \|^2 + \frac{(L+\beta)\gamma^2}{2(m-b)} \sum_{i=1}^{m-b} \E\| v_{(i)} \|^2 \\
&\leq -\frac{\gamma}{2} \| \nabla F(x) \|^2 + \frac{\gamma}{2} \E \| \nabla F(x) - \bar{v} \|^2 + \frac{(L+\beta)\gamma^2 G}{2}.
\end{align*}

Now we just need to bound $\E \| \nabla F(x) - \bar{v} \|^2$. We define that $\mathcal{S}_1 = \{\nabla F_i(x): i \in [m]\} \setminus  \{v_{(i)}: i \in [m-b]\}$. Note that $|\mathcal{S}_1| = b$.
\begin{align*}
&\E \| \nabla F(x) - \bar{v} \|^2 \\
&= \E \left\| \frac{1}{m} \sum_{i \in [m]} \E \left[ \nabla F_i(x) \right] - \frac{1}{m-b} \sum_{i=1}^{m-b} v_{(i)} \right\|^2 \\
&\leq 2E \left\| \frac{1}{m} \sum_{i \in [m]} \E \left[ \nabla F_i(x) \right] - \frac{1}{m} \sum_{i=1}^{m-b} v_{(i)} \right\|^2 + 2E \left\| \frac{1}{m} \sum_{i=1}^{m-b} v_{(i)} - \frac{1}{m-b} \sum_{i=1}^{m-b} v_{(i)} \right\|^2 \\
&\leq 4E \left\| \frac{1}{m} \sum_{i \in [m]} \E \left[ \nabla F_i(x) \right] - \frac{1}{m} \sum_{i=1}^{m}  \nabla F_i(x) \right\|^2 + 4E \left\| \frac{1}{m} \sum_{v \in \mathcal{S}_1} v \right\|^2 + 2E \left\| \frac{1}{m} \sum_{i=1}^{m-b} v_{(i)} - \frac{1}{m-b} \sum_{i=1}^{m-b} v_{(i)} \right\|^2 \\
&\leq \frac{4V}{m} + \frac{4b^2}{m^2} E \left\| \frac{1}{b} \sum_{v \in \mathcal{S}_1} v \right\|^2 + 2E \left\| \frac{1}{m} \sum_{i=1}^{m-b} v_{(i)} - \frac{1}{m-b} \sum_{i=1}^{m-b} v_{(i)} \right\|^2 \\
&\leq \frac{4V}{m} + \frac{4b^2}{m^2} \frac{mG}{b} + 2 \left[ \frac{1}{m} - \frac{1}{m-b} \right]^2 E \left\| \sum_{i=1}^{m-b} v_{(i)} \right\|^2 \\
&\leq \frac{4V}{m} + \frac{4bG}{m} + 2 \left[ \frac{b}{m(m-b)} \right]^2 (m-b)(m-q)G \\
&\leq \frac{4V}{m} + \frac{4bG}{m} + \frac{2b^2(m-q)G}{m^2(m-b)}.
\end{align*}

Thus, we have 
\begin{align*}
&\E \left[ F(x - \gamma \bar{\tilde{v}}) \right] - F(x) \\
&\leq -\frac{\gamma}{2} \| \nabla F(x) \|^2 + \frac{\gamma}{2} \left[ \frac{4V}{m} + \frac{4bG}{m} + \frac{2b^2(m-q)G}{m^2(m-b)} \right] + \frac{(L+\beta)\gamma^2 G}{2}.
\end{align*}

Follow the same procedure in Corollary~\ref{cor:fixed_lr}, taking $\gamma = \frac{1}{L \sqrt{T}}$, we have
\begin{align*}
&\frac{\sum_{t=0}^{T-1} \E \| \nabla F(x^t) \|^2}{T} \\
&\leq \frac{2L \left[ F(x^0) - F(x^*) \right]}{\sqrt{T}} +  \frac{4V}{m} + \frac{4bG}{m} + \frac{2b^2(m-q)G}{m^2(m-b)} + \frac{(L + \beta)G}{L \sqrt{T}} \\
&= \mathcal{O}\left( \frac{1}{\sqrt{T}} \right) + \mathcal{O}\left(\frac{b}{m}\right) + \mathcal{O}\left(\frac{b^2(m-q)}{m^2(m-b)}\right).
\end{align*}

\end{proof}

\section{Additional Experiments}

\begin{figure*}[htb!]
\centering
\subfigure[Top-1 accuracy on testing set, with $q=8$]{\includegraphics[width=0.49\textwidth]{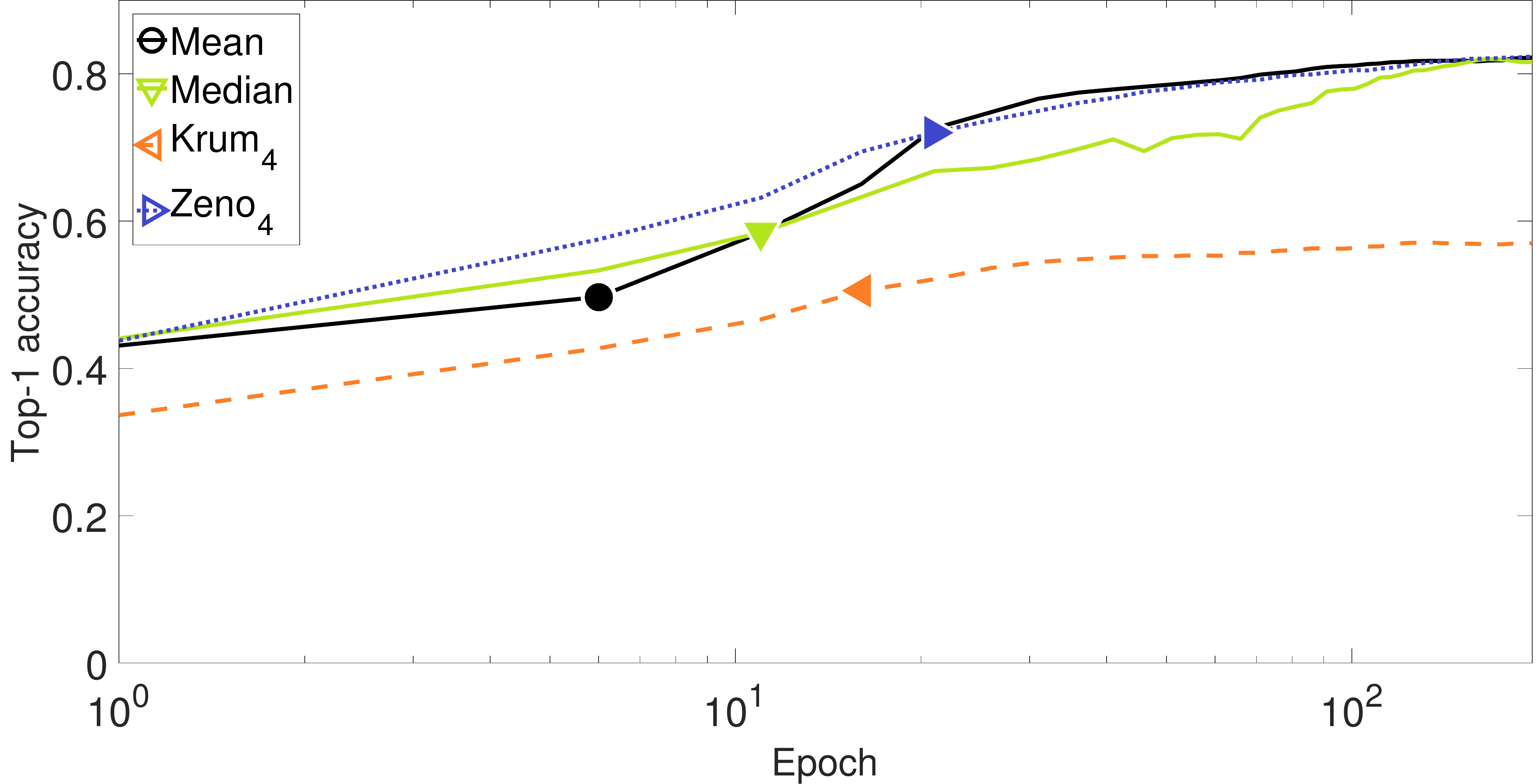}}
\subfigure[Cross entropy on training set, with $q=8$]{\includegraphics[width=0.49\textwidth]{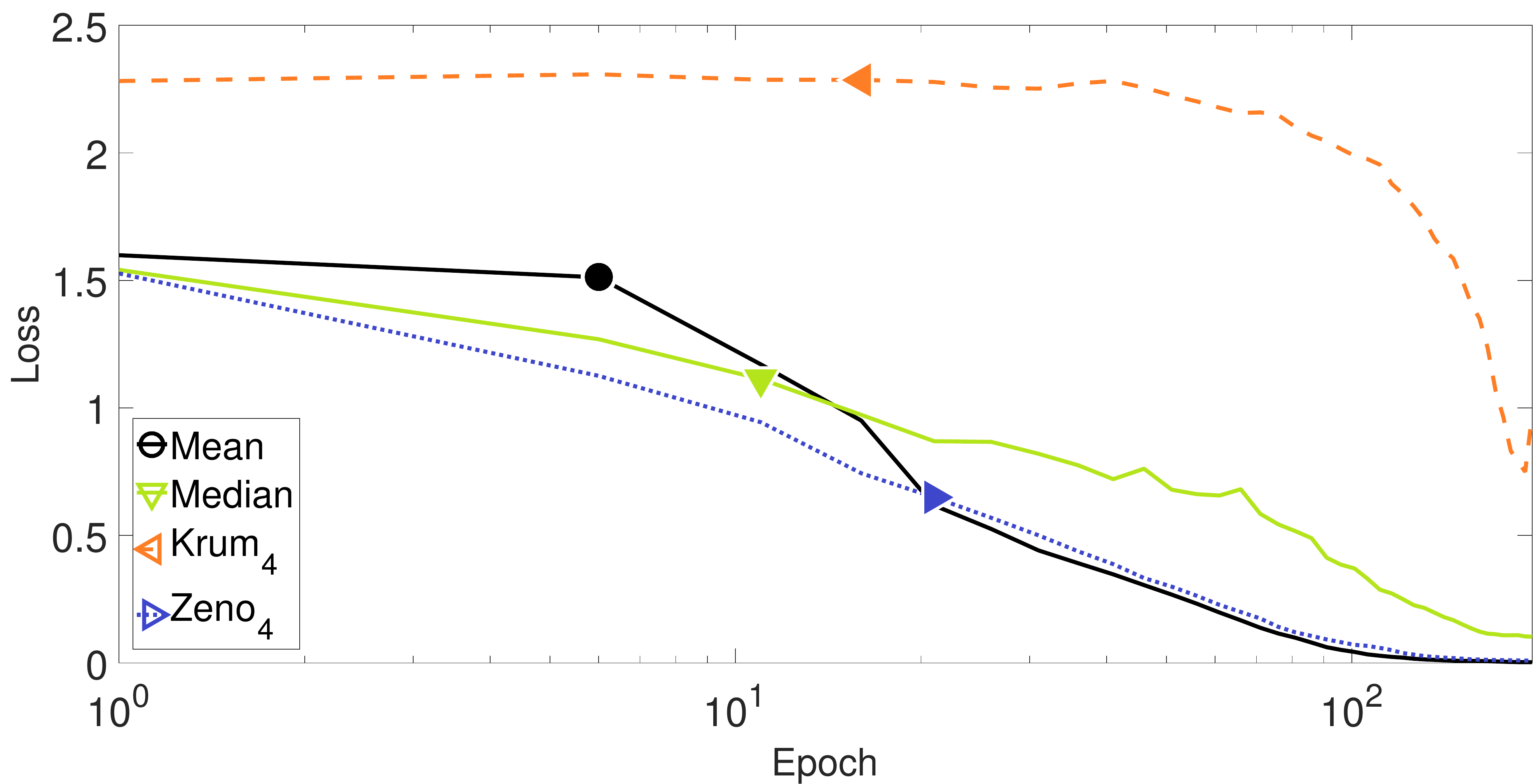}}
\caption{Convergence on non-i.i.d. training data, without failures. Batch size on the workers is $100$. Batch size of \texttt{Zeno} is $n_r=4$. $\rho=0.0005$. Learning rate $\gamma = 0.05$.  Each epoch has 25 iterations.}
\label{fig:noniid_nobyz}
\end{figure*}

\begin{figure*}[htb!]
\centering
\subfigure[Top-1 accuracy on testing set, with $q=8$]{\includegraphics[width=0.49\textwidth]{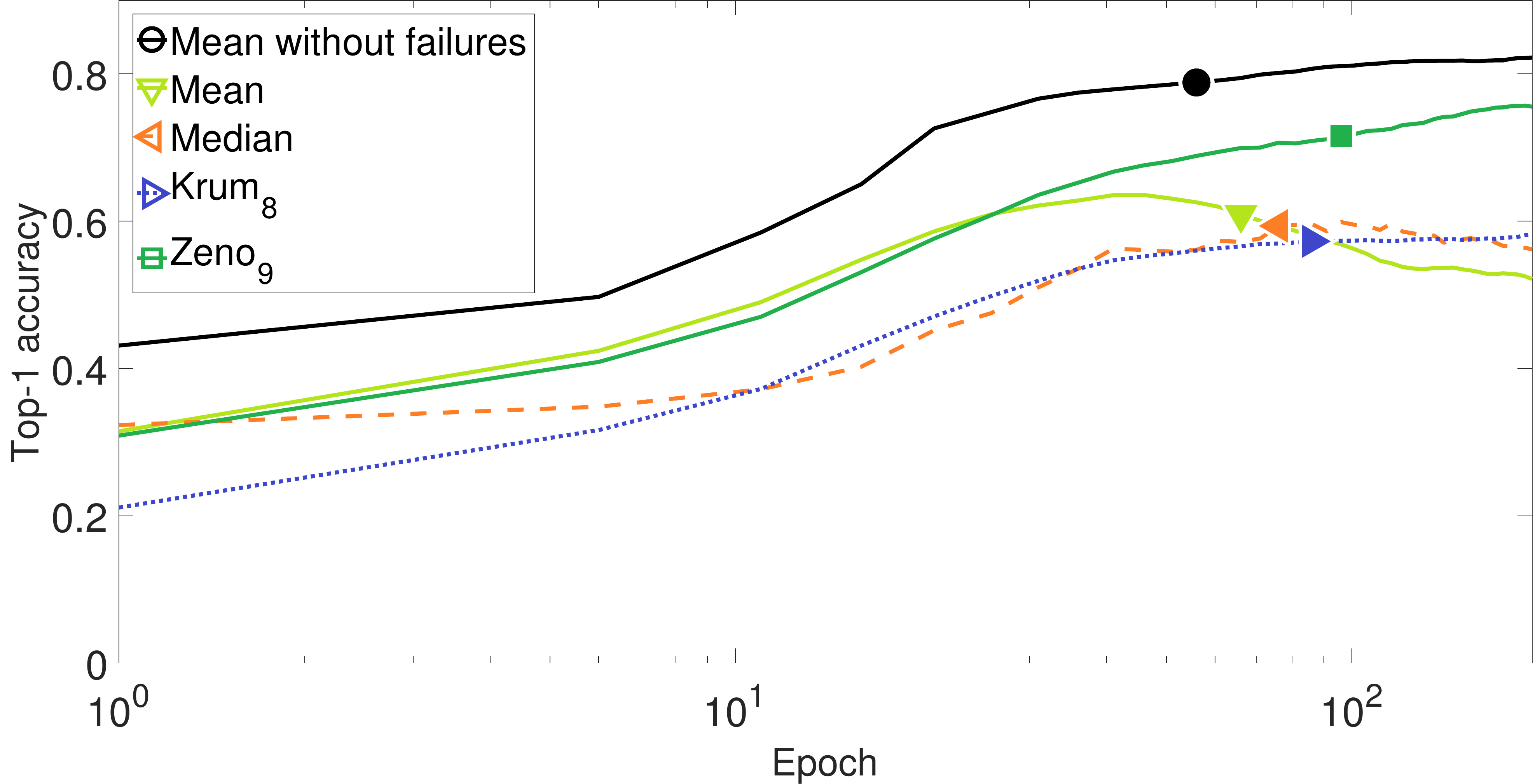}}
\subfigure[Cross entropy on training set, with $q=8$]{\includegraphics[width=0.49\textwidth]{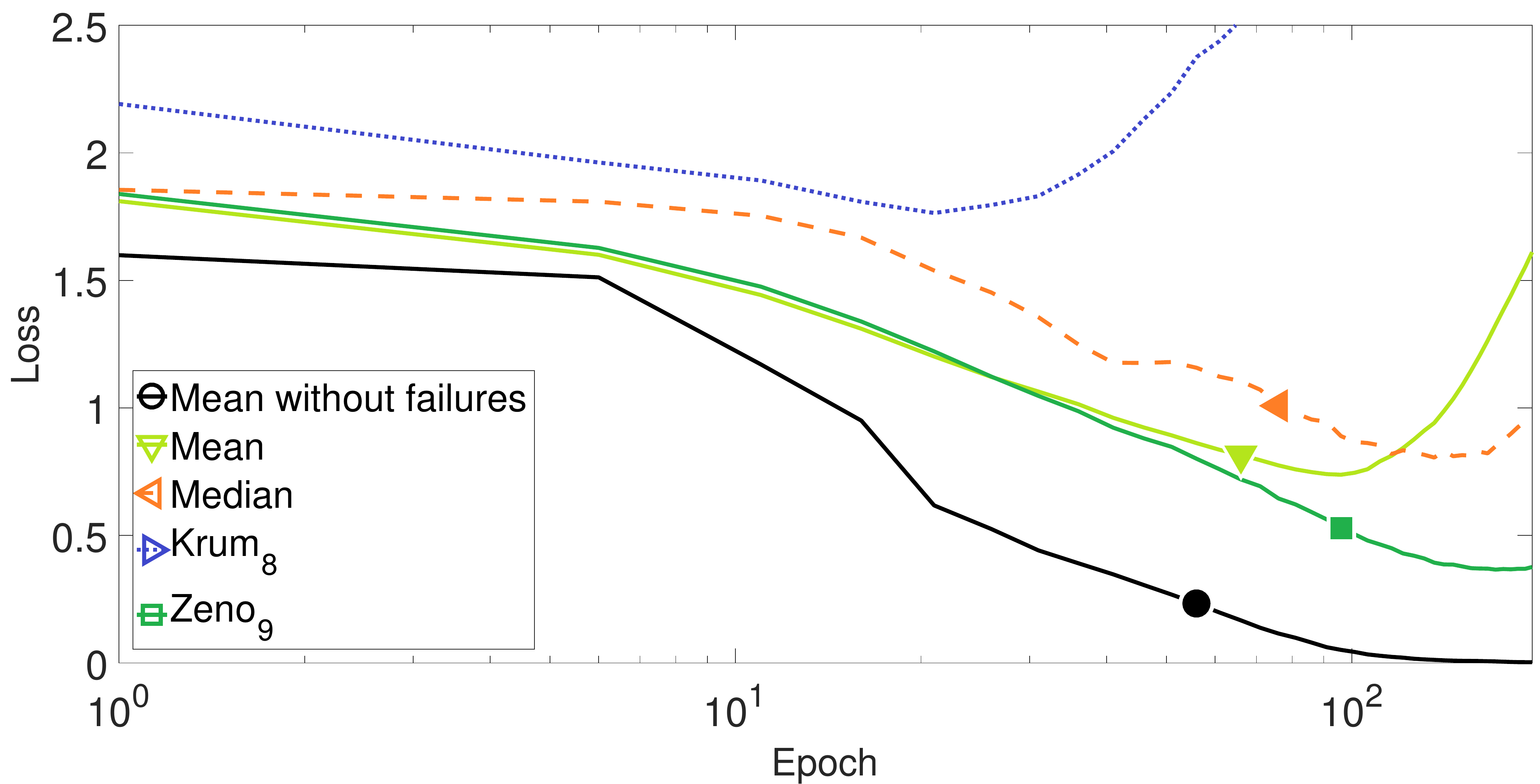}}
\subfigure[Top-1 accuracy on testing set, with $q=12$]{\includegraphics[width=0.49\textwidth]{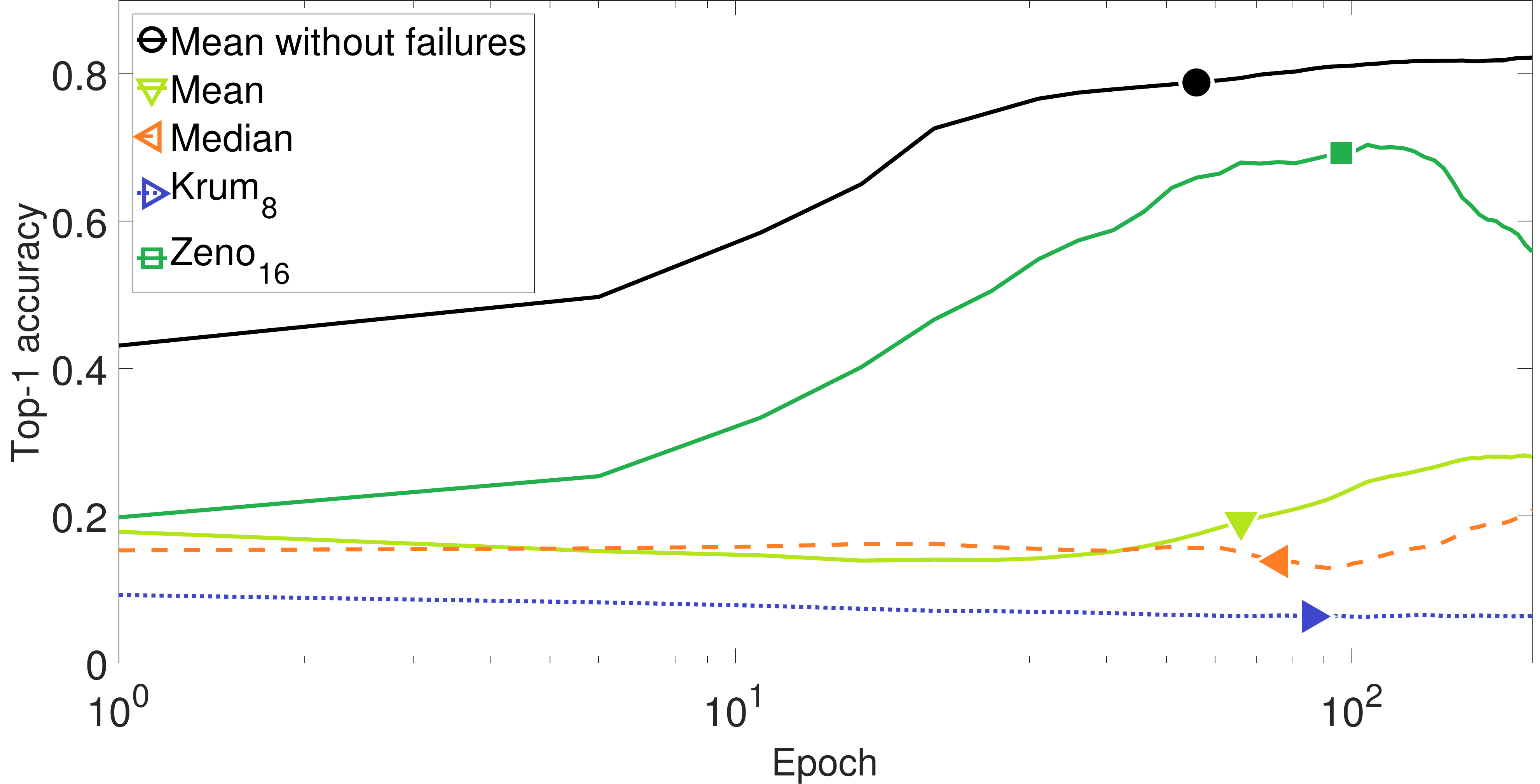}}
\subfigure[Cross entropy on training set, with $q=12$]{\includegraphics[width=0.49\textwidth]{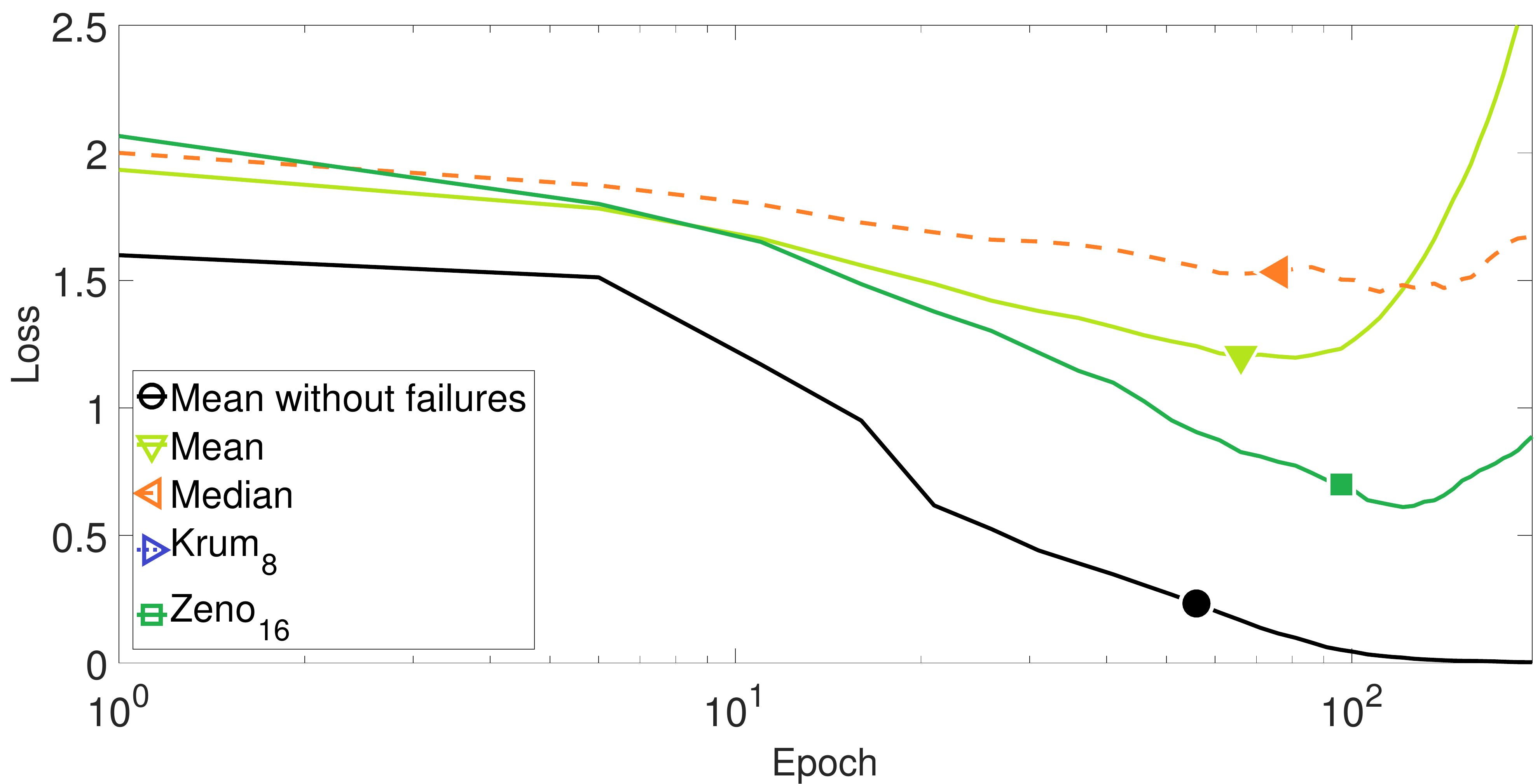}}
\caption{Convergence on non-i.i.d. training data, with label-flipping failures. Batch size on the workers is $100$. Batch size of \texttt{Zeno} is $n_r=4$. $\rho=0.0005$. Learning rate $\gamma = 0.05$.  Each epoch has 25 iterations.}
\label{fig:noniid_labelflip_2}
\end{figure*}

\begin{figure*}[htb!]
\centering
\subfigure[Top-1 accuracy on testing set, with $q=8$]{\includegraphics[width=0.49\textwidth]{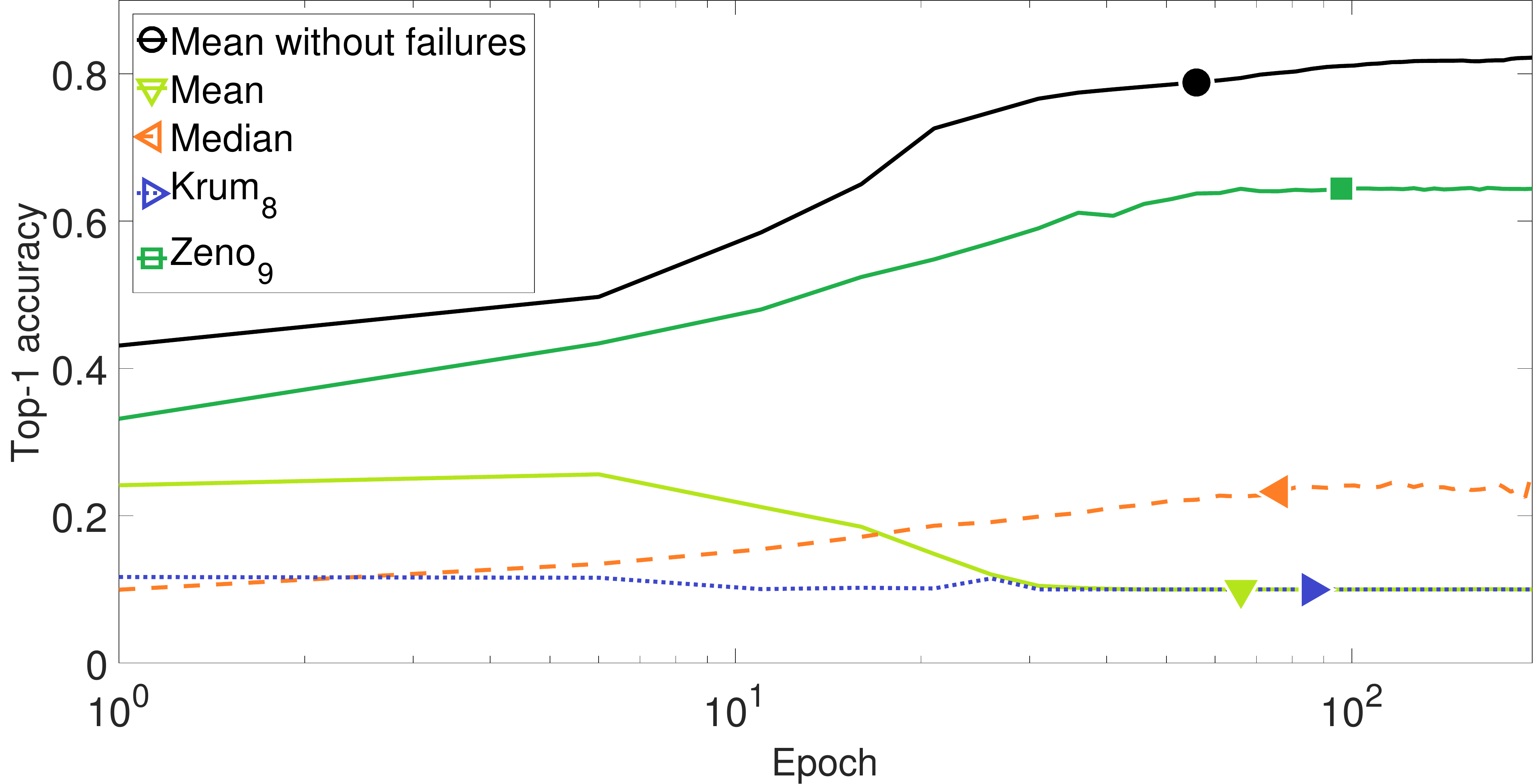}}
\subfigure[Cross entropy on training set, with $q=8$]{\includegraphics[width=0.49\textwidth]{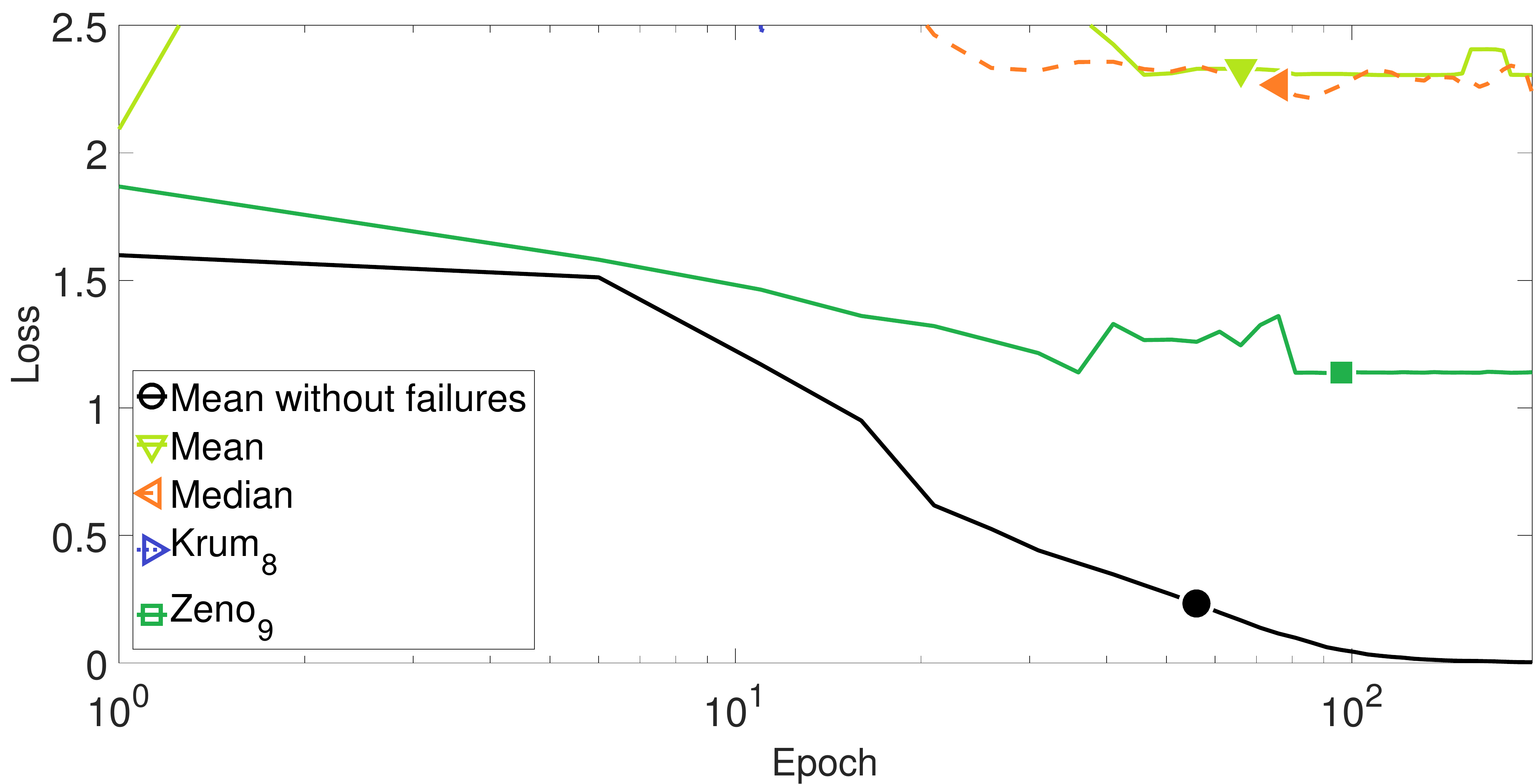}}
\subfigure[Top-1 accuracy on testing set, with $q=12$]{\includegraphics[width=0.49\textwidth]{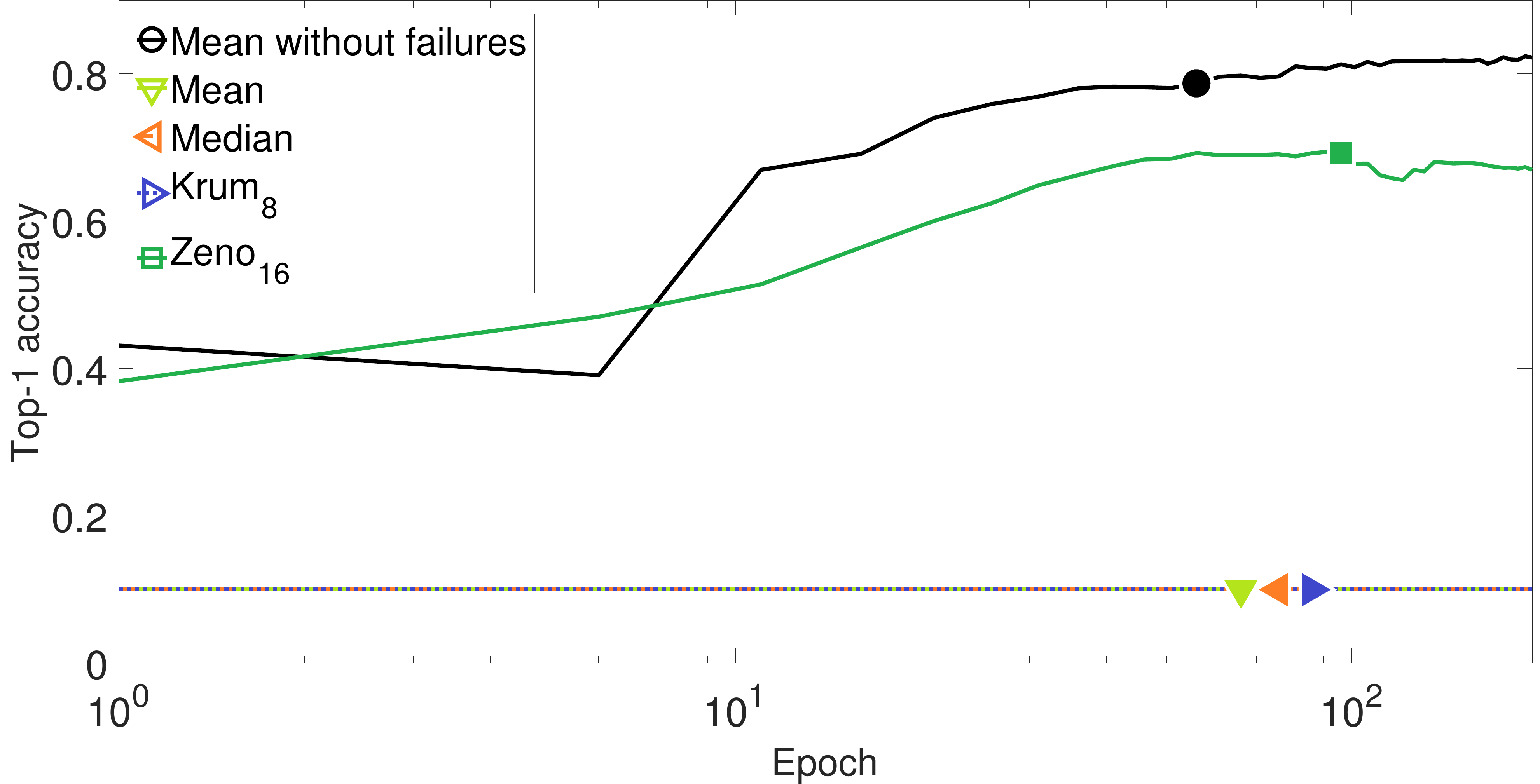}}
\subfigure[Cross entropy on training set, with $q=12$]{\includegraphics[width=0.49\textwidth]{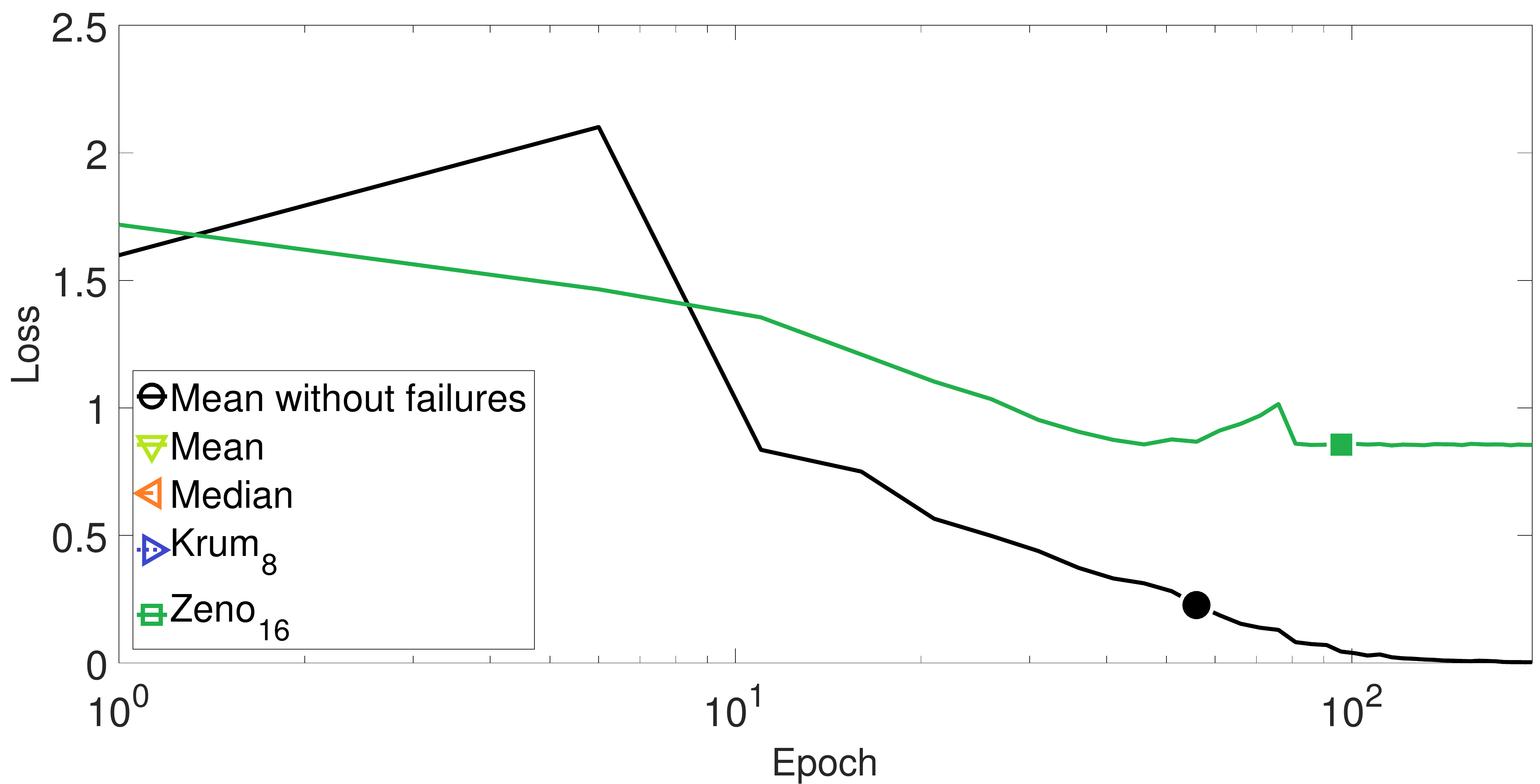}}
\caption{Convergence on non-i.i.d. training data, with bit-flipping failures. Batch size on the workers is $100$. Batch size of \texttt{Zeno} is $n_r=4$. $\rho=0.0005$. Learning rate $\gamma = 0.05$.  Each epoch has 25 iterations.}
\label{fig:noniid_signflip}
\end{figure*}

\begin{figure*}[htb!]
\centering
\subfigure[Top-1 accuracy on testing set, with $q=8$]{\includegraphics[width=0.49\textwidth]{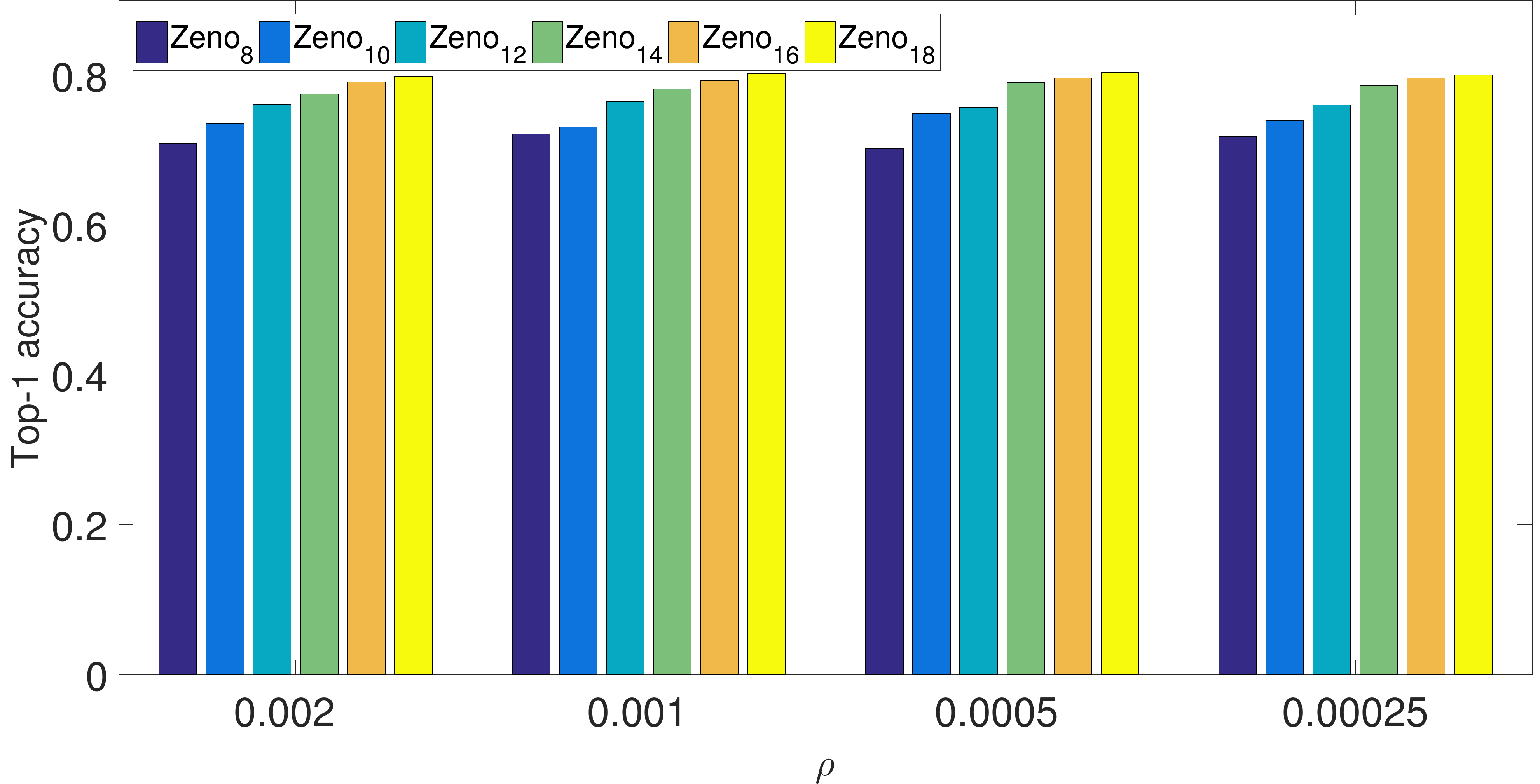}}
\subfigure[Cross entropy on training set, with $q=8$]{\includegraphics[width=0.49\textwidth]{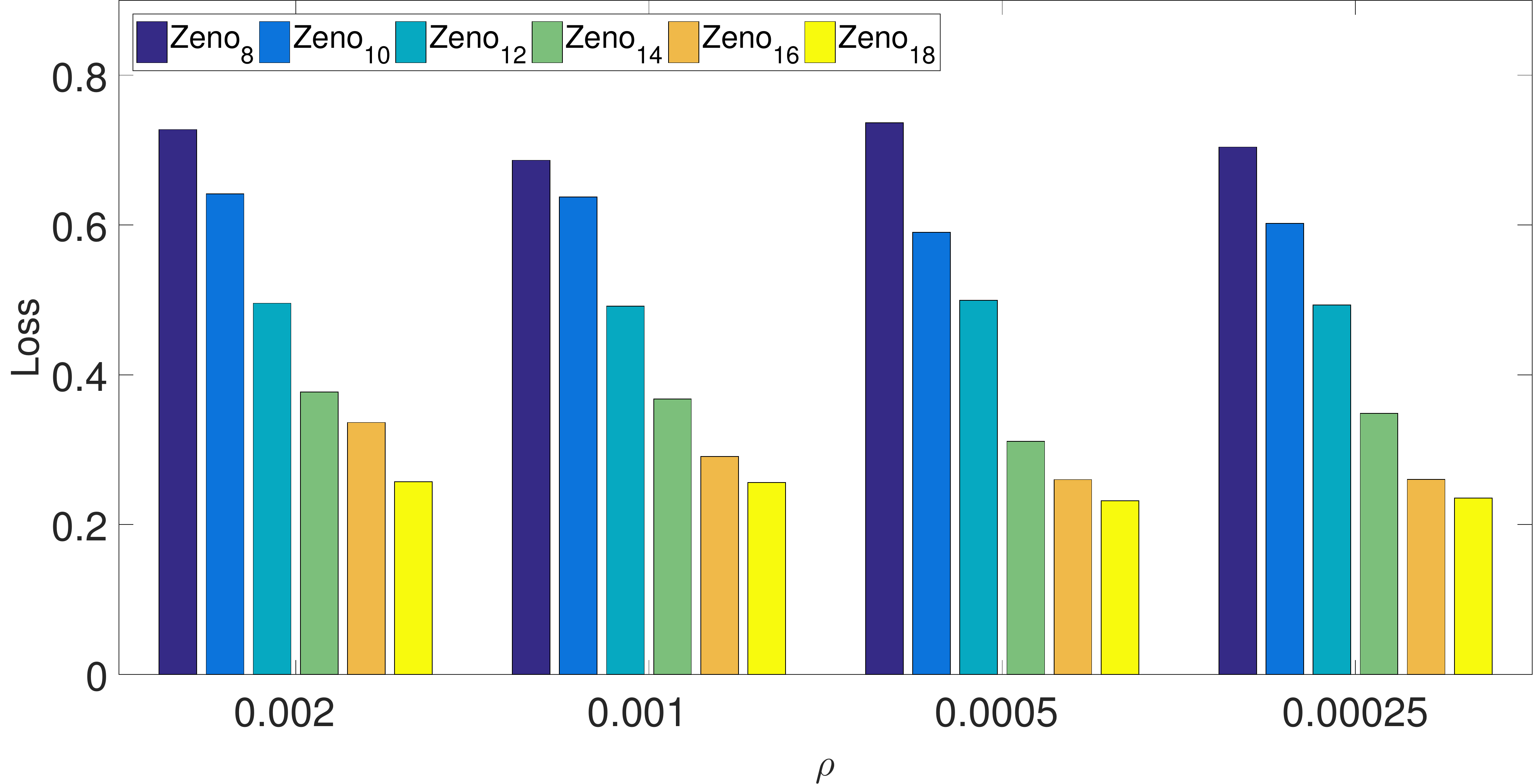}}
\caption{Convergence on i.i.d. training data, with label-flipping failures, $q=8$. Batch size on the workers is $100$. Batch size of \texttt{Zeno} is $n_r=4$. Learning rate $\gamma = 0.1$.  Each epoch has 25 iterations. $\rho$ and $b$ are tuned.}
\label{fig:rho_b}
\end{figure*}

\end{document}